\documentclass[10pt,letterpaper]{article}
\usepackage[utf8]{inputenc}
\usepackage{amsthm,amssymb,latexsym,amsmath,graphics,graphicx,tabularx,caption,subcaption,mathtools}
\usepackage[usenames,dvipsnames]{color}
\usepackage{algorithm,algorithmic,multirow}
\usepackage[shortlabels]{enumitem}
\usepackage[left=2cm,right=2cm,top=2cm,bottom=2cm]{geometry}
\usepackage{hyperref}
\hypersetup{
  colorlinks   = true,    
  urlcolor     = blue,    
  linkcolor    = blue,    
  citecolor    = red      
}
\usepackage{nomencl}
\makenomenclature

\theoremstyle{plain}
\newtheorem{theorem}{Theorem}[section]

\newtheorem{corollary}{Corollary}[section]
\newtheorem{lemma}{Lemma}[section]
\newtheorem{proposition}{Proposition}[section]
\theoremstyle{definition}
\newtheorem{definition}{Definition}[section]
\newtheorem{example}{Example}[section]
\newtheorem{rem}{Remark}[section]

\newcommand\nc[2]{\nomenclature{$#1$}{#2}}
\newcommand\norm[1]{\left|\left|#1\right|\right|}
\newcommand\abs[1]{\left|#1\right|}
\newcommand\alg[1]{\left\langle#1\right\rangle}
\newcommand\floor[1]{\left\lfloor#1\right\rfloor}

\newcommand{\XX}{\mathbb{X}}
\renewcommand{\SS}{\mathbb{S}}
\newcommand{\RR}{\mathbb{R}}
\newcommand{\w}{\mathbf{w}}
\newcommand{\x}{\mathbf{x}}

\newcommand{\be}{\begin{equation}}
\newcommand{\ee}{\end{equation}}
\newcommand{\ba}{\begin{aligned}}
\newcommand{\ea}{\end{aligned}}
\newcommand{\ben}{\begin{enumerate}}
\newcommand{\een}{\end{enumerate}}
\newcommand{\bit}{\begin{itemize}}
\newcommand{\eit}{\end{itemize}}
\newcommand{\ls}{\lesssim}
\newcommand{\gs}{\gtrsim}
\newcommand{\disp}{\displaystyle}
\newcommand{\bs}{\boldsymbol}
\renewcommand{\vec}[1]{\mathbf{#1}}
\newcommand{\hindu}{\arabic}

\title{Learning on manifolds without manifold learning}
\author{H. N. Mhaskar\thanks{
Institute of Mathematical Sciences, Claremont Graduate University, Claremont, CA 91711. 
\textsf{email:} hrushikesh.mhaskar@cgu.edu.
The research is  supported in part by NSF grant DMS 2012355, and ONR grants N00014-23-1-2394, N00014-23-1-2790.} \and Ryan O'Dowd\thanks{Institute of Mathematical Sciences, Claremont Graduate University, Claremont, CA 91711. 
\textsf{email:} ryan.o'dowd@cgu.edu.}}
\date{\today}

\numberwithin{equation}{section}

\begin{document}
\maketitle


\begin{abstract}
Function approximation based on data drawn randomly from an unknown distribution is an important problem in machine learning.
The manifold hypothesis assumes that
 the data is sampled from an unknown submanifold of a high dimensional Euclidean space. 
A great deal of research deals with obtaining information about this manifold, such as the eigendecomposition of the Laplace-Beltrami operator or coordinate charts, and using this information for function approximation. 
This two-step approach implies some extra errors in the approximation stemming from estimating the basic quantities of the data manifold in addition to the errors inherent in function approximation.
In this paper, we project the unknown manifold as a submanifold of an ambient hypersphere and study the question of constructing a one-shot approximation using a specially designed sequence of localized spherical polynomial kernels  on the hypersphere. 
Our approach does not require preprocessing of the data to obtain information about the manifold other than its dimension. 
We give optimal rates of approximation for relatively ``rough'' functions.

\end{abstract}

\section{Introduction}
\label{sec:intro}

In the past quarter-century, machine leaning has impacted our lives ubiquitously, from driving cars to military maneuvers. 
Shallow and deep neural networks have played a central role in these applications. 
In turn, a theoretical justification for the use of these networks is their universal approximation property: they can approximate arbitrarily well an arbitrary continuous function on an arbitrary compact subset of a Euclidean space of arbitrary dimension.
In mathematical terms, the challenge can be formulated as follows. 
We are given a data of the form $\mathcal{D}=\{(y_j, z_j)\}_{j=1}^M$, drawn randomly from an unknown probability distribution $\tau$,  and we wish to find a parametrized model $G(\bs\theta;y)$ to minimize the \emph{generalization error} $\mathbb{E}_\tau(\mathcal{L}(z,G(\bs\theta,y)))$ for a judiciously chosen loss functional $\mathcal{L}$. 
For example, in a shallow neural network of the form $\sum_{k=1}^m a_k\sigma(\w_k\cdot\x+b_k)$, $\x\in\RR^q$, the parameters are $\bs\theta=(a_k, w_{k,1},\cdots, w_{k,q}, b_k)_{k=1}^m$.
Since $\tau$ is not known, one minimizes instead the empirical risk obtained by discretizing the expected value in terms of the data. 
There is a huge amount of literature on the choice of the loss functional, usually involving the correct choice of one or more regularization terms, the difference between the minimal empirical risk and the generalization error in terms of the number of samples, strategies for the optimization involved, the geometry of the error surface, etc.

Writing $f(y)=\mathbb{E}_\tau(z|y)$, the fundamental problem is to approximate $f$ given the data $\mathcal{D}$.
The role of approximation theory is to estimate the relation between the minimum generalization error and number of parameters in $\bs\theta$  in terms of some properties of $f$ and $G$. 
Naturally, there is a huge amount of literature in this direction as well, especially when $\mathcal{L}$ is a square loss or, since $\tau$ is unknown, the uniform or probabilistic loss. 
In the case of the square loss, the generalization error splits into the variance $\mathbb{E}_\tau(|z-f(y)|^2)$ and the bias $\mathbb{E}_{\nu}(|f(y)-G(\bs\theta,y)|^2)$, where $\nu$ is the marginal distribution of $y$. 

We think that the whole paradigm of getting an insight into the number of parameters using approximation theory, and then using an optimization procedure to actually obtain the approximation in a decoupled manner needs to be revisited. We list some reasons.

\begin{enumerate}
\item The use of a global metric for measuring the generalization/approximation error is insensitive to local effects in the target function (see Example~\ref{ex:nswvsphin}).
\item The use of the degree of approximation to get an insight on the model complexity may be misleading, as we will elaborate on shortly.
\item There is no guarantee that the minimizer of the empirical risk would be the one which gives the best approximation error or generalization error.
Moreover, absolute minima are hard to obtain, and perhaps not required in practice anyway.
\item The training process may be very sensitive to the initialization of the parameters. It is observed in \cite{lu2019dying} that with a wrong initialization of parameters, a deep network evaluating the ReLU activation function trained to approximate $|t|=t_++(-t)_+$ results in a constant output. This is a phenomenon which they have called ``dead on arrival.''
\end{enumerate}

Since the problem is fundamentally one of function approximation, it is natural to question whether one could use a new paradigm where the approximation is constructed directly from the data, and the error on the data not yet seen can be estimated directly as well.
So far, approximation theory has played only a marginal role in machine learning.
There are several reasons for this.
\begin{enumerate}
\item Many papers on function approximation by shallow or deep networks ignore the fact that the approximation needs to be constructed from the data. For example, the dimension independent bounds are typically derived using probabilistic arguments resulting in estimates which could be misleading. 
Particularly, we have shown in \cite{mhaskar2020dimension, sphrelu} that drastically different estimates are obtained for approximation by  ReLU networks for the \textbf{same class of functions} depending on whether the networks are constructed from the data or not.
\item We do not typically know whether the assumptions on the target function involved   in the approximation theory bounds are satisfied in practice, or whether the number of parameters is the right criterion to look at in the first place.
For example, when one considers approximation by radial basis function (RBF) networks, it is  observed in many papers (e.g., \cite{eignet}) that the minimal separation among the centers is the right criterion rather than the number of parameters. 
It is shown that if one measures the degree of approximation in terms of the minimal separation, then one can determine the smoothness of the underlying target function by examining the rate at which the degrees of approximation converge to $0$.
\item Most of the approximation theory literature focuses on the question of estimating the difference between $f$ and $G(\bs\theta,\circ)$ in various norms and conditions on $f$, where the support of the marginal distribution $\nu$ is assumed to be a known domain, such as a torus, a cube, the whole Euclidean space, a hypersphere (simply referred to as a sphere in the remainder of this paper), etc.; equivalently, one assumes that the data points $y_j$ are ``dense'' on such a domain.
This creates a gap between theory, where the domain of $\nu$ is known, and practice, where it is not.
One consequence of approximating, say on a cube, is the curse of dimensionality. That is, if the dimension of the input data is $Q$, then the number of components in the parameter vector $\bs\theta$ to achieve an accuracy of $\epsilon$ will be $\Omega(\epsilon^{-Q})$.
\end{enumerate}

Rather than approximating on a known domain, a relatively recent idea is to assume that the support of the marginal distribution $\nu$ is an \textbf{unknown}, low-dimensional submanifold of the high-dimensional ambient space in which the data is located.
This gives rise to a two-step procedure: \emph{manifold learning}, where there is an effort to find information about the manifold itself, and then \emph{function approximation} (which we have called \emph{learning on the manifold} in the title of this paper), where we assume the necessary information about the manifold to be known, and study function approximation based on this information. 
 
Works by Belkin, Niyogi,  Singer, and others have shown that the so-called graph Laplacian (and the corresponding eigendecomposition) constructed from data points converges to the manifold Laplacian and its eigendecomposition. 
Some preliminary papers in this direction are: \cite{belkin2003laplacian,belkinfound,singerlaplacian}.
An introduction to the subject is given in \cite{achaspissue}.  Another approach is to estimate an atlas of the manifold, which thereby allows function approximation to be conducted via local coordinate charts. 
One such effort is to utilize the underlying parametric structure of the functions to determine the dimension of the manifold and the parameters involved \cite{manoni2020effective}.
Approximations utilizing estimated coordinate charts have been implemented, for example, via deep learning  \cite{coifman_deep_learn_2015bigeometric,schmidtdeep}, moving least-squares  \cite{soberleastsquares}, local linear regression  \cite{Wuregression},  or using Euclidean distances among the data points \cite{chui_deep}. HNM and his collaborators carried out an extensive investigation of function approximation on manifolds, some of which is summarized in \cite{mhaskardata}.
With the two-step procedure, the estimates obtained in function approximation need to be tempered by the errors accrued in the manifold learning step.
In turn, the errors in the manifold learning step are very sensitive to the choice of different parameters used in the process.

The purpose of this paper is to introduce a direct method of approximation on \emph{unknown} manifolds without trying to find out anything about the manifold other than its dimension. 
Toward this goal, we project   the $q$-dimensional manifold $\XX$ in question from the ambient space $\RR^Q$ to a sphere $\SS^Q$ of the same dimension.
We can then use a specially designed, localized, univariate kernel $\Phi_{n,q}$  (cf. \eqref{eq:kernel}) which is a spherical polynomial of degree $< n$ on $\SS^Q$, with $n$ and $q$ being tunable hyperparameters. 
Our construction is very simple; we define
\begin{equation}\label{eq:proto_const}
    F_{n}(\mathcal{D};x)\coloneqq \frac{1}{M}\sum_{j=1}^M z_j\Phi_{n,q}(x\cdot y_j).
\end{equation}
We note that $F_n(\mathcal{D};\circ)$ is a function defined on the ambient sphere $\SS^Q$.
The localization of the kernel allows us to adapt the approximation to the unknown manifold.

Our main theorem (cf. Theorem~\ref{theo:mainthm}) has the following form:

\begin{theorem}\label{theo:proto_mainthm} (\textbf{Informal statement})
Let $\mathcal{D}=\{(y_j,z_j)\}_{j=1}^M$ be a set of random samples chosen from a distribution $\tau$. Suppose $f$ belongs to a smoothness class $W_\gamma$ (detailed in Definition~\ref{def:manifold_smoothness}) with associated norm $\norm{\circ}_{W_\gamma}$. Then under some additional conditions and with a judicious choice of $n$, we have with high probability:
\begin{equation}\label{eq:proto_est}
    \norm{F_n(\mathcal{D};\circ)-f}_\mathbb{X}\leq c\left(\norm{z}+\norm{f}_{W_\gamma}\right)\left(\frac{\log M}{M}\right)^{\gamma/(q+2\gamma)},
\end{equation}
where $c$ is a positive constant independent of $f$.
\end{theorem}

We note some mathematical features of our construction and theorem which we find interesting.
\begin{enumerate}
\item The usual approach in machine learning is to construct the approximation using an optimization procedure, usually involving a regularization term.
The setting up of this optimization problem, especially the regularization term, requires one to assume that the function belongs to some special function class, such as a reproducing kernel Hilbert/Banach space.
Thus, the constructions are not explicit nor universal.
In contrast, our construction \eqref{eq:proto_const} does not require  a prior on the function in order to use our model.
Of course, the theorem and its high-probability convergence rates do require various assumptions on $\tau$, the marginal distribution, the dimension of the manifold, the smoothness of the target function, etc. 
The point is that the construction itself does not require any assumptions.
\item 
A major problem in manifold learning is one of out of sample extension; i.e., extending the approximation to outside the manifold. 
A usual procedure for this in the context of approximation using the eigenstructure of the Laplace-Beltrami operator on the manifold is the Nystr\"om extension \cite{coifman2006geometric}.
However, this extension is no longer in terms of any orthogonal system on the ambient space, and hence there is no guarantee of the quality of approximation even if the function is known outside the manifold.
In contrast, the point $x$ in \eqref{eq:proto_const} is not restricted to the manifold, but rather freely chosen from $\mathbb{S}^Q$. 
That is, our construction defines an out of sample extension in terms of spherical polynomials on the ambient sphere, whose approximation capabilities are well studied.
\item In terms of $M$, the estimate in \eqref{eq:proto_est} depends asymptotically on the dimension $q$ of the manifold rather than the dimension $Q$ of the ambient space.
\item We do not need to know \textbf{anything} about the manifold (e.g., eigendecomposition or atlas estimate) itself apart from its dimension in order to prove our theorem.
There are several papers in the literature for estimating the dimension from the data, for example \cite{liao2016learning, liao2016adaptive, manoni2020effective}. 
However, the simplicity of our construction allows us to treat the dimension $q$ as a tunable parameter to be determined by the usual division of the data into training, validation, and test data.
\end{enumerate}

There are several other approaches superficially similar to our constructions. 
We will comment on some of these in Section~\ref{sec:related}.
We describe the main idea behind our proofs in Section~\ref{sec:overview}.
The paper requires an understanding of the approximation properties of spherical polynomials. Accordingly, we describe some background on the spherical polynomials, our localized kernels, and their use in approximation theory on subspheres of the ambient sphere
in Section~\ref{sec:background}. 
The main theorems for approximation on the unknown manifold are given in Section~\ref{sec:manifoldapprox}. 
The theorems are illustrated with three numerical examples in Section~\ref{sec:numerical}. 
One of these examples is closely related to an important problem in magnetic resonance relaxometry, in which one seeks to find the proportion of water molecules in the myelin covering in the brain based on a model that involves inversion of the Laplace transform.
The proofs of the main theorems are given in Section~\ref{sec:proofs}. 
The appendix describes the encoding of the target function (\ref{sec:encoding}), gives some background about the theory of manifolds which is used in this paper (\ref{sec:manifoldintro}), and describes in detail the Clenshaw algorithm used to evaluate our kernels and their implementation as a deep neural network (\ref{subsec:clenshaw}).

We would like to thank Dr. Richard Spencer at the National Institute of Aging (NIH) for his helpful comments, especially on Section~\ref{subsec:spencerdata}, verifying that our simulation is consistent with what is used in the discipline of magnetic resonance relaxometry.

\section{Related ideas}\label{sec:related}

Since our method is based on a highly localized kernel, it is expected to be comparable to the simple nearest neighbor algorithm.
However, rather than specifying the number of neighbors to consider in advance, our method allows the selection of neighbors adaptively  for each test point, controlled by the parameter $n$.
Also, rather than taking a simple averaging, our method is more sophisticated, designed to give an optimal order of magnitude of the approximation error.

One of the oldest ideas for data based function approximation is the so-called Nadaraya-Watson estimator (NWE), given by
$$
NW_h(x)=\frac{\sum_{j=1}^M z_jK(|x-y_j|/h)}{\sum_{j=1}^M K(|x-y_j|/h)},
$$
where $K$ is a kernel with an effectively small support---the Gaussian kernel $K(t)=\exp(-t^2)$, as a common example---and $h$ is a scaling parameter.
Another possible choice is a $B$-spline (including Bernstein polynomials) which has a compact support.
This construction is designed to work on a Euclidean space by effectively shrinking the support of $K$ using the scaling parameter $h\to 0$, analogously to spline approximation.
The degree of approximation of such methods is measured in terms of $h$.
It is well known (e.g., \cite{de2006approximation}) that  the use of a positive kernel $K$  suffers from the so-called saturation phenomenon: the degree of approximation cannot be smaller than $\mathcal{O}(h^2)$ unless the function is a trivial one in some sense.

Radial basis function (RBF) networks and neural networks are used widely for function approximation, using either interpolation or least square fit. 
Standard RBF networks, such as Gaussian networks or thin plate spline networks, use a fixed, scaled kernel.
Typically, the matrices involved in either interpolation or least square approximation are very ill-conditioned, and the approximation is not highly localized. 

Restricted to the sphere, both of the notions are represented by a zonal function (ZF) network. 
A \emph{zonal function} on a sphere is a function of the form $x\mapsto g(x\cdot x_0)$. 
A ZF network is a linear combination of finitely many zonal functions.
One may notice that
$$
g(x\cdot x_0)=g\left(1-\frac{|x-x_0|^2}{2}\right),
$$
so we can see that a ZF network is also a neural/RBF network. 
Conversely, a neural/RBF network restricted to the sphere is a ZF network. 
The same observations about RBF networks hold for ZF networks as well.
We note that all the papers we are aware of which deal with approximation by ZF networks actually end up approximating a spherical polynomial by the networks in question.

Rather than working with a fixed, scaled kernel, in this paper we deal with a sequence of highly localized polynomial kernels.
We do not need to solve any system of equations or do any optimization to arrive at our construction.
RBF networks and NWE were developed for approximation on Euclidean domains instead of unknown manifolds. Both have a single hyperparameter $h$ and work analogously to the spline approximation.
In contrast, our method is designed for approximation on unknown manifolds without having to learn anything about the manifold besides the dimension. It has two integer hyperparameters ($n$ and $q$) and yields a polynomial approximation.

If one chooses $h$ small enough relative to a fixed $n$ then NWE may be able to outperform our method as measured in terms of a global error bound, such as the root mean square (RMS) error. If one instead chooses $n$ large enough relative to a fixed $h$ then our method may be able to outperform NWE. So in order to give a fair comparison in Example~\ref{ex:nswvsphin}, we force the RMS error of both methods to be approximately equivalent and investigate the qualitative differences of the errors produced by each method. We additionally show that both methods in the example outperform an interpolatory RBF network.

\begin{example}\label{ex:nswvsphin}
This example serves to illustrate two points.
The first point is to compare the performance our method with NWE and an RBF interpolant. In doing so, we show that the error associated with our method is localized to singularities of the target function, whereas the other methods do not exhibit this behavior.
The second point is that using a global error estimate such as RMS can be misleading. 
Even if the RMS error with a given method might be large, the percentage of test data points at which it is smaller than a threshold could be substantially higher due to the local effects in the target function.

To ensure fair comparison, we use each of the three methods for approximation on $\SS^1=\{(\cos\theta,\sin\theta): \theta\in (-\pi,\pi]\}$, where the Gaussian kernel can be expressed in the form of a zonal function as explained above.

We consider the function
\be\label{eq:toyexample}
f(\theta)=1+\abs{\cos\theta}^{7/2}\sin(\cos\theta+\sin\theta)/2, \qquad \theta\in (-\pi,\pi].
\ee
We note that the function is analytic except at $\theta=\pm\pi/2$, where it has a discontinuity in the 4th order derivative. 
Our training data consists of $2^{13}$ equidistantly spaced $y_j$'s along the circle. We set $z_j=f(y_j)$, and examine the resulting error on a test data consisting of $2^{11}$ points chosen randomly according to the uniform distribution on $\SS^1$.  

We consider three approximation processes : (1)  Nadaraya-Watson estimator NW$_h$ with\\
 $K_h(t)=\exp(-t^2/h^2)$, (2) interpolatory approximation by the RBF network of the form $\sum a_k\exp(-|\circ-y_j|^2/h^2)$, (3) our method with the kernel $\Phi_{50,1}$.

We experimentally determined the optimal $h$ value in NWE to be $\approx 7.45e\text{-4}$ (effectively simulating the minimization of the actual generalization error on the test data), resulting in an RMS error of 1.8462e-7. 
The same value of $h$ was used for interpolation with the Gaussian RBF network, yielding a RMS error of 2.2290e-4. 
We then chose $n$ so as to yield a (comparable to NWE) RMS error of 1.8594e-7 (though we note that in this case our method continues to provide a better approximation if $n$ is further increased).

The detailed results are summarized in Figure~\ref{fig:nwecomp} below.
\begin{figure}[!ht]
\centering
\begin{tabular}{cc}
    \includegraphics[width=.45\textwidth]{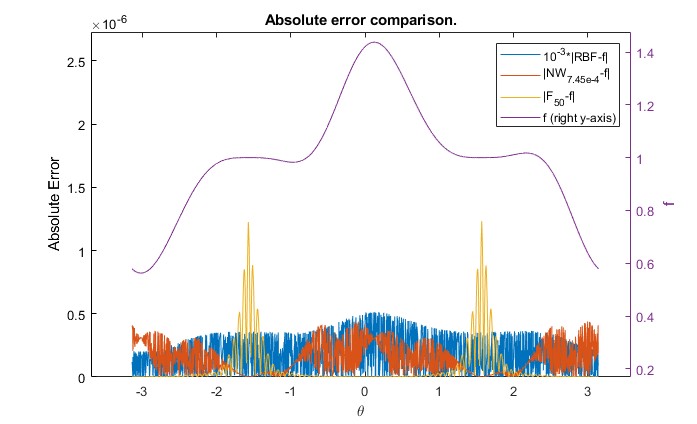}
    &
    \includegraphics[width=.45\textwidth]{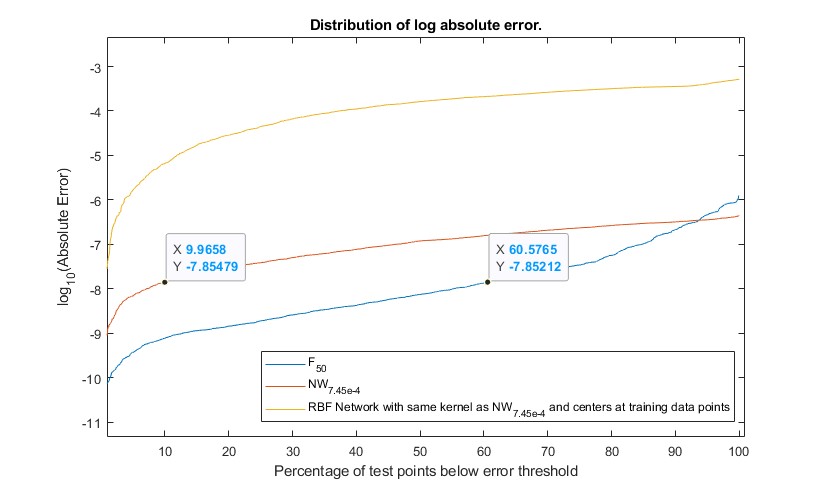}
\end{tabular}
    \caption{Error comparison between our method, the Nadaraya-Watson estimator, and an interpolatory RBF network. (Left) Comparison of absolute errors between the methods with the target function plotted on the right $y$-axis for benefit of the viewer. We note that the error from the RBF method is scaled by $10^{-3}$ so as to not dominate the figure. (Right) Percent point plot of the log absolute error for all three methods.}
        \label{fig:nwecomp}

\end{figure}

In the left plot in Figure~\ref{fig:nwecomp},  we can see a clear difference between the errors of the three methods. The (scaled by $10^{-3}$) error from the RBF network jumps throughout the whole domain, signaling the ill-conditioned nature of the matrix.
The error from the Nadaraya-Watson estimator exhibits some oscillation across the whole domain as well. The error with our method is localized to the two singularity points of the function. 
In other words, our method exhibits 1) sensitivity to the singularities of a function and 2) error adapting to the local smoothness of the function.
In comparison, RBF networks and NWE do not always exhibit such behavior. 
On the right plot of Figure~\ref{fig:nwecomp}, we give a percent point plot of the log absolute error for all three methods. There are three curves corresponding to the three methods being compared. Each point $(x,y)$ along a given curve indicates that the corresponding method approximated $x\%$ of test points with absolute error below $10^{y}$. This plot can also be thought of as the inverse CDF for the random variable of the resulting log absolute error for a test point sampled uniformly at random.
For example, whereas the Nadaraya-Watson estimator yields an error below $\approx 10^{-7.85}$ for only about $10\%$ of the tested points, our method exhibits the same error or below for about $60\%$ of the test points. Our method has the higher uniform error, but lower error for over $90\%$ of the test points.
Although the overall RMS error is roughly the same, our method exhibits a quicker decay from the uniform error.
This illustrates, in particular, that measuring the performance using a global measure for the error, such as the uniform or RMS error can be misleading.
The interpolatory RBF network performs the worst of the three methods as the right plot of Figure~\ref{fig:nwecomp} shows clearly. \qed
\end{example}

There are some efforts \cite{fuselier2012scattered,lehmann2019ambient} to do function approximation on manifolds using tensor product splines or RBF networks defined on an ambient space by first extending the target function to the ambient space.
A locally adaptive polynomial approach is used in \cite{sober2017approximation} for accomplishing function approximation on manifolds using the data. 
All these works require that the manifold be known.

In \cite{mhaskar2020deep}, we have suggested a direct approach to function approximation on an unknown submanifold of a Euclidean space using a localized kernel based on Hermite polynomials.
This construction was used successfully in predicting diabetic sugar episodes \cite{gaussian_diabetes} and recognition of hand gestures \cite{mason2021manifold}.
In  particular, in \cite{gaussian_diabetes}, we constructed our approximation based on one clinical data set and used it to predict the episodes based on another clinical data set.
In order to extend the applicability of such results to wearable devices, it is important that the approximation should be encoded by a hopefully small number of real numbers, which can then be hardwired or used for a simpler approximation process \cite{valeriyasmartphone}.
However, the construction in \cite{mhaskar2020deep} is a linear combination of kernels of the form $\Psi(|\circ-y_j|)$, where $\Psi(t)=P(t)\exp(-t^2/2)$ is a univariate kernel utilizing a judiciously chosen polynomial $P$. 
This means that we get a good approximation, but the space from which the approximation takes place changes with the point at which the approximation is desired.
This does not allow us to encode the approximation using finitely many real numbers.
In contrast, the method proposed in this paper allows us to encode the approximation using  coefficients of the target function in the spherical harmonic expansion (defined in a distributional sense), computed empirically.
This sequence can be reduced using connections between ultraspherical polynomials with different parameters and simple algorithms to detect and remove redundant coefficients. 
This is described in Appendix~\ref{sec:encoding}.
Moreover, the degree of the polynomials involved in  \cite{mhaskar2020deep} to obtain same the rate of convergence in terms of the number of samples is $\mathcal{O}(n^2)$, while the degree of the polynomials involved in this paper is $\mathcal{O}(n)$. 
We note that the construction in both the papers involve only univariate polynomials, so that the dimension of the input space enters only linearly in the complexity of the construction.

\section{An overview of the proof}\label{sec:overview}

\begin{figure}[!ht]
\centering
    \includegraphics[width=0.4\textwidth]{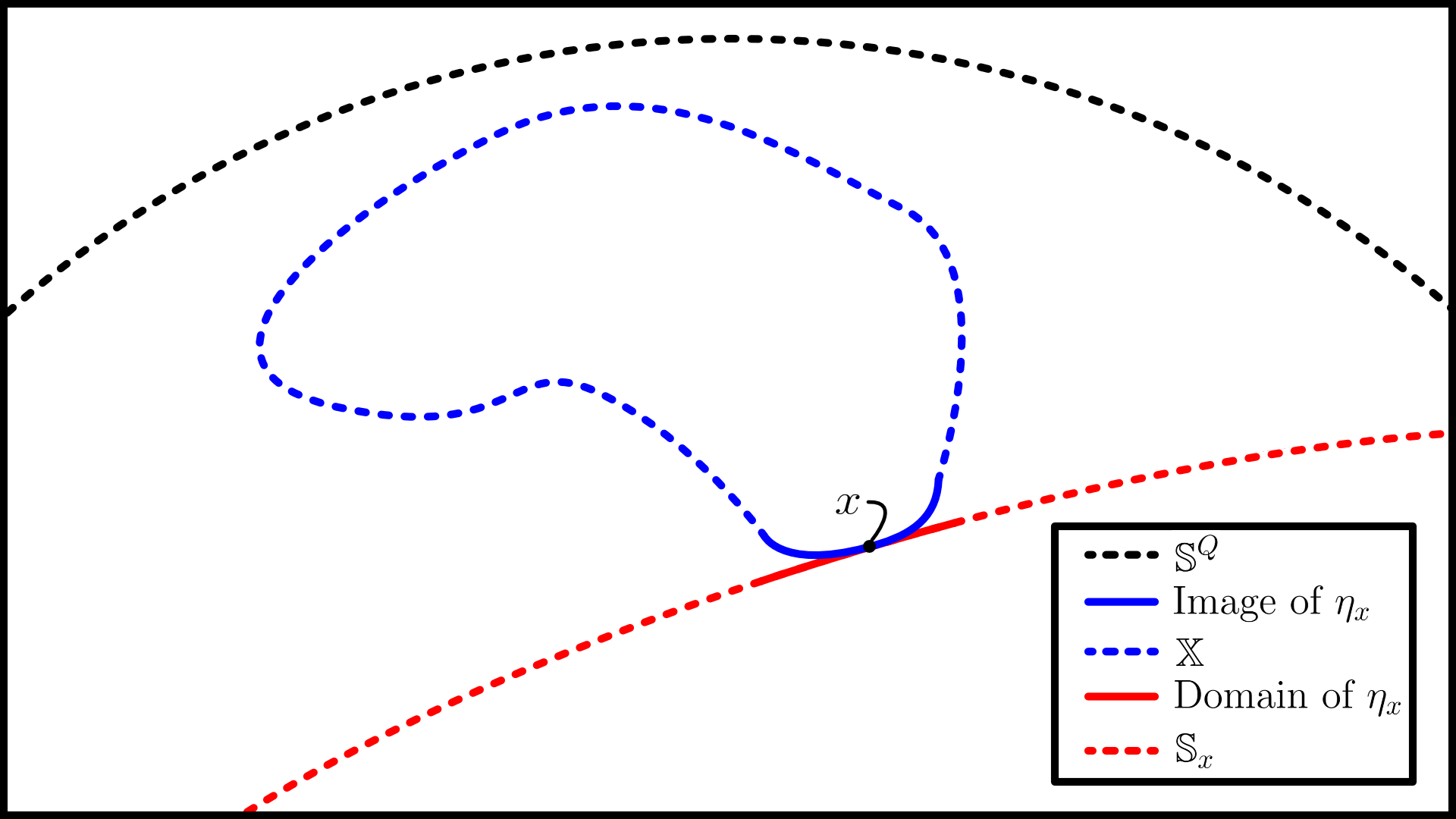}
    \caption{Visualization of our approximation approach. Here, $\mathbb{X}$ is a submanifold of the sphere $\mathbb{S}^Q$. The map $\eta_x$, analogous to the exponential map, allows us to relate the part of the integral in \eqref{eq:manifold_summabilityop} near $x$ with an integral on the tangent sphere at $x$ via a change of variables (solid curves). The localization of the kernels in our method allow for the approximation to be extended over $\XX$ and the tangent sphere $\mathbb{S}_x$ (dotted curves).}
        \label{fig:method}

\end{figure}

We can think of $F_n(\mathcal{D};x)$ defined in \eqref{eq:approximation} as an empirical approximation for an expected value with respect to the data distribution $\tau$:
\be\label{eq:approx_vs_expect}
\mathbb{E}_\tau(F_n(\mathcal{D};x))=\int z\Phi_{n,q}(x\cdot y)d\tau(y). 
\ee
Assuming that the marginal distribution of $\tau$ on $\XX$ is absolutely continuous with respect to the Riemannian volume measure $\mu^*$ on  $\XX$; i.e., given by $f_0d\mu^*$ for some smooth function $f_0$, we have
\be\label{eq:expect_vs_volumemeasure}
\mathbb{E}_\tau(F_n(\mathcal{D};x))=\int_\XX f(y)f_0(y)\Phi_{n,q}(x\cdot y)d\mu^*(y).
\ee
Accordingly, we define an \emph{integral reconstruction operator} by
\be\label{eq:manifold_summabilityop}
\sigma_n(\mathbb{X},f)(x)\coloneqq \int_{\mathbb{X}}\Phi_{n,q}(x\cdot y)f(y)d\mu^*(y), \qquad f\in C(\mathbb{X}), \ x\in\XX,
\ee
study the approximation properties of this operator, and use it with $ff_0$ in place of $f$.
The approximation properties of the operator $\sigma_n$ in the case of when $\XX$ is the $q$-dimensional sphere $\SS^q$ are well known (Proposition~\ref{prop:sphapprox}), and can be easily transferred to a $q$-dimensional equator of the ambient sphere $\SS^Q$ (Section~\ref{subsec:eqapprox}, Theorem~\ref{theo:subthm}).
We introduce a local exponential map $\eta_x$ at $x\in\XX$ between $\XX$ and the tangent equatorial sphere $\SS_x$ (i.e., a rotated version of $\SS^q$ that shares the tangent space with $\XX$ at $x$). We give an illustration of this setup in Figure~\ref{fig:method}.
Locally, a change of variable formula and the properties of this map allow us to compare the integral over a small manifold ball with that of its image on $\SS_x$ (cf. \eqref{eq:lemma_7_1_main_est}).
We keep track of the errors using the Bernstein inequality for spherical polynomials (cf. \eqref{eq:sph_bernstein}) and standard approximations between geodesic distances and volume elements on the manifold by those on $\mathbb{S}_x$.
This constitutes the main part of the proof of the critical Lemma~\ref{lemma:manifold_keylemma}.
We use the high localization property of our kernel $\Phi_{n,q}$ to lift the rest of the integral in \eqref{eq:manifold_summabilityop} on $\mathbb{X}$ at any point $x\in\XX$ to the rest of $\SS_x$ with small error (cf. \eqref{eq:lemma_7_1_awaymanifold},~\eqref{eq:lemma_7_1_awaysphere}).   
After this, we can use known results from the theory of approximation on the sphere by spherical polynomials (cf. Proposition~\ref{prop:sphapprox} and Theorem~\ref{theo:subthm}).
A partition of unity argument is used often in the proof.

Having obtained the approximation result for the integral reconstruction operator, we then discretize the integral and keep track of the errors using concentration inequalities.

\section{Background}
\label{sec:background}

In this section, we outline some important details about spherical harmonics (Section~\ref{subsec:sphharm}) which leads to the construction of the kernels of interest in this paper (Section~\ref{subsec:LocalizedKernels}). We then review some classical approximation results using these kernels on spheres (Section~\ref{subsec:sphapprox}) and equators of spheres (Section~\ref{subsec:eqapprox}).

\subsection{Spherical harmonics}
\label{subsec:sphharm}
The material in this section is based on \cite{mullerbk, steinweissbk}.
Let $0\le q\le Q$ be integers. We define the \textit{$q$-dimensional sphere} embedded in $Q+1$-dimensional space as follows
\begin{equation}\label{eq:sphere}
    \mathbb{S}^q\coloneqq\{(x_1,\dots,x_{q+1},\underbrace{0,\dots,0}_{Q-q}): x_1^2+\cdots+x_{q+1}^2=1\}.
\end{equation}
Observe that $\mathbb{S}^q$ is a $q$-dimensional compact manifold with geodesic defined by $\rho(x,y)=\arccos(x\cdot y)$. 

 Let $\mu_q^*$ denote the normalized volume measure on $\mathbb{S}^q$. 
 By representing a point $x\in \mathbb{S}^q$ as $(x' \sin\theta,\cos\theta)$ for some $x'\in \mathbb{S}^{q-1}$, one has the recursive formula for measures
\be\label{eq:recursivemeasure}
\frac{\omega_q}{\omega_{q-1}} d\mu^*_q(x)=\sin^{q-1}(\theta)d\theta d\mu^*_{q-1}(x'),
\ee
where $\omega_q$ denotes the surface volume of $\mathbb{S}^q$. One can write $\omega_q$ recursively by
\be\label{eq:volumerecurs}
\omega_{q}= \frac{2\pi^{(q+1)/2}}{\Gamma((q+1)/2)}=
\begin{cases}
2\pi, &\mbox{ if $q=1$},\\[1ex]
\displaystyle\sqrt{\pi}\frac{\Gamma(q/2)}{\Gamma(q/2+1/2)}\omega_{q-1}, &\mbox{ if $q\ge 2$,}
\end{cases}
\ee
where $\Gamma$ denotes the Gamma function. 

The restriction of a homogenous harmonic polynomial in $q+1$ variables to the $q$-dimensional unit sphere $\mathbb{S}^q$
is called a spherical harmonic. 
The space of all spherical harmonics of degree $\ell$ in $q+1$ variables will be denoted  by $\mathbb{H}_\ell^q$. 
The  space of the restriction of  all $q+1$ variable polynomials of degree $< n$ to $\mathbb{S}^q$ will be denoted by $\Pi_n^q$. 
We extend this notation for an arbitrary real value $x>0$ by writing  $\Pi_x^q\coloneqq \Pi_{\floor{x}}^q$.
It is known that $\mathbb{H}_\ell^q$ is orthogonal to $\mathbb{H}_j^q$ in $L^2(\mu_q^*)$ whenever $j\neq \ell$, and $\disp\Pi_n^q=\bigoplus_{\ell=0}^{n-1   }\mathbb{H}_\ell^q$.
In particular, $\disp L^2(\mu^*_q)=\bigoplus_{\ell=0}^{\infty}\mathbb{H}_\ell^q$.

 If we let $\{Y_{\ell,k}\}_{k=1}^{\dim (\mathbb{H}_\ell^q)}$ be an orthonormal basis for $\mathbb{H}_\ell^q$ with respect to $\mu^*_q$, we can define
\begin{equation}\label{eq:reproduce}
    K_{q,\ell}(x,y)\coloneqq \sum_{k=1}^{\dim (\mathbb{H}^q_\ell)} Y_{\ell,k}(x)Y_{\ell,k}(y).
\end{equation}
In \cite{mullerbk, steinweissbk}, it is shown that 
\be\label{eq:summation}
K_{q,\ell}(x,y)=\frac{\omega_{q}}{\omega_{q-1}}p_{q,\ell}(1)p_{q,\ell}(x\cdot y),
\ee
where $p_{q,\ell}$ denotes the orthonormalized ultraspherical polynomial of dimension $q$ and degree $\ell$. These ultraspherical polynomials satisfy the following orthogonality condition.
\begin{equation}\label{j-orth}
    \int_{-1}^1 (1-x^2)^{(q/2-1)} p_{q,m}(x)p_{q,n}(x)dx=\delta_{m,n}.
\end{equation}
Computationally, it is customary to use the following recurrence relation:
\begin{equation}\label{eq:recurrence}
\begin{aligned}
\sqrt{\frac{(n+1)(n+q-1)}{(2n+q-1)(2n+q+1)}}&p_{q,n+1}(x)=xp_{q,n}(x)-\sqrt{\frac{n(n+q-2)}{(2n+q-1)(2n+q-3)}}p_{q,n-1}(x), \qquad n\ge 1,\\
& p_{q,0}(x)=p_{q,0}=2^{1/2-q/2}\frac{\Gamma(q-1)}{\Gamma(q/2)},\qquad p_{q,1}(x)=2^{1/2-q/2}\frac{\sqrt{\Gamma(q)\Gamma(q+1)}}{\Gamma(q/2)}x.
\end{aligned}
\end{equation}
We note further that
\begin{equation}\label{eq:pnat1}
    p_{q,n}(1)=\frac{2^{1/2-q/2}}{\Gamma(q/2)}\sqrt{\frac{\Gamma(n+q-1)(2n+q-1)}{\Gamma(n+1)}}.
\end{equation}

\begin{rem}{\rm
Many notations have been used for ultraspherical polynomials in the past. 
For example, \cite{szegopoly} uses the notation of $P_n^{(\lambda)}$ for the Gegenbauer polynomials, also commonly denoted by $C_n^{(\lambda)}$. 
It is also usual to use a normalization, which we will denote by $R_n^q$ in this remark, given by $R_n^q=p_{q,n}/p_{q,n}(1)$.
Ultraspherical polynomials are also simply a special case of the Jacobi polynomials $P_n^{(\alpha,\beta)}$ where $\alpha=\beta$. 
Setting
\be\label{eq:normalization_coeff}
h_{q,n}\coloneqq 2^{q-1}\frac{\Gamma(n+q/2)^2}{n!\Gamma(n+q-1)(2n+q-1)},
\ee
 we have the following connection between these notations:
\be\label{eq:notation_relation}
\ba
p_{q,n}(x)&=h_{q,n}^{-1/2}P_{n}^{(q/2-1,q/2-1)}(x)=\frac{\Gamma(q-1)}{\Gamma(q/2)}\sqrt{\frac{n!(2n+q-1)}{2^{q-1}\Gamma(n+q-1)}}C_n^{(q/2-1/2)}(x)\\
&=\frac{2^{1/2-q/2}}{\Gamma(q/2)}\sqrt{\frac{\Gamma(n+q-1)\Gamma(2n+q-1)}{\Gamma(n+1)}}R_n^q.
\ea
\ee
\qed}
\end{rem}

Furthermore, the ultraspherical polynomials for the sphere of dimension $d_1$ can be represented by those for the sphere of dimension $d_2$ in the following manner
\be\label{eq:dimchange}
p_{d_1,n}=\sum_{\ell=0}^n C_{d_2,d_1}(\ell,n)p_{d_2,\ell}.
\ee
The coefficients $C$ have been studied, and explicit formulas are given in \cite[Equation 7.34]{askeypoly} and \cite[Equation 4.10.27]{szegopoly}.

\textbf{The constant convention}\\
\textit{In the sequel, $c, c_1,\cdots$ will denote generic positive constants depending upon the fixed quantities in the discussion, such as the manifold, the dimensions $q$, $Q$, and various parameters such as $S$ to be introduced below. 
Their values may be different at different occurrences, even within a single formula.
The notation $A\ls B$ means $A\le cB$, $A\gs B$ means $B\ls A$, and $A\sim B$ means $A\ls B\ls A$.}

\subsection{Localized kernels}
\label{subsec:LocalizedKernels}

Let $h$ be an infinitely differentiable function supported on $[0,1]$ where $h(x)=1$ on $[0,1/2]$. 
This function will be fixed in the rest of this paper, and its mention will be omitted from the notation.
Then we define the following univariate kernel for $t\in [-1,1]$:
\begin{equation}\label{eq:kernel}
\Phi_{n,q}(t)\coloneqq \Phi_{n,q}(h;t)=\sum_{\ell=0}^n h\left(\frac{\ell}{n}\right)K_{\ell,q}(t)=\frac{\omega_q}{\omega_{q-1}}\sum_{\ell=0}^{n}h\left(\frac{\ell}{n}\right)p_{q,\ell}(1)p_{q,\ell}(t).
\end{equation}
The following proposition lists some technical properties of these kernels which we will often use, sometimes without an explicit mention.
\begin{proposition}\label{prop:kernelprop}
Let $x,y\in \mathbb{S}^Q$. For any $S>0$, the kernel $\Phi_{n,q}(x,y)$ satisfies the \textbf{localization bound}
\begin{equation}
\label{eq:sphkernloc}
    |\Phi_{n,q}(x\cdot y)|\lesssim \frac{n^q}{\max(1,n\arccos(x\cdot y))^S},
\end{equation}
where the constant involved may depend upon $S$.
Further, we have the Lipschitz condition:
\begin{equation}
\label{eq:sph_bernstein}
|\Phi_{n,q}(x\cdot y)-\Phi_{n,q}(x\cdot y')|\lesssim n^{q+1}|\arccos(x\cdot y)-\arccos(x\cdot y')|,\qquad y'\in \mathbb{S}^Q.
\end{equation}
%
\end{proposition}

\begin{proof}
The estimate \eqref{eq:sphkernloc} is proved in \cite{mhaskarpoly}.
Since $\theta\mapsto \Phi_{n,q}(\cos\theta)$ is a trigonometric polynomial of degree $< n$, the Bernstein inequality for the derivatives of trigonometric polynomials implies that
$$
|\Phi_{n,q}(\cos\theta)-\Phi_{n,q}(\cos\phi)|\le n\|\Phi_{n,q}\|_\infty |\theta-\phi|\ls n^{q+1}|\theta-\phi|.
$$
This leads easily to \eqref{eq:sph_bernstein}.
%
\end{proof}

\subsection{Approximation on spheres}
\label{subsec:sphapprox}

Methods of approximating functions on $\mathbb{S}^q$ have been studied in, for example, \cite{mhaskarsphere,rustamovsphere} and some details are summarized in Proposition~\ref{prop:sphapprox}. 

For a compact set $A$, let $C(A)$ denote the space of continuous functions on $A$, equipped with the supremum norm $\norm{f}_{A}=\max_{x\in A}|f(x)|$. 
We define the degree of approximation for a function $f\in C(\mathbb{S}^q)$ to be
\begin{equation}\label{eq:degapprox}
E_n(f)\coloneqq\inf_{P\in \Pi_n^q} \norm{f-P}_{\mathbb{S}^q}.
\end{equation}
Let $W_\gamma(\mathbb{S}^q)$ be the class of all $f\in C(\mathbb{S}^q)$ such that
\begin{equation}\label{eq:sphsoboldef}
||f||_{W_\gamma(\mathbb{S}^q)}\coloneqq ||f||_{\mathbb{S}^q}+\sup_{n\geq 0}2^{n\gamma}E_{2^n}(f)<\infty.
\end{equation}
We note that an alternative smoothness characterized in terms of constructive properties of $f$ is explored by many authors; some examples are given in  \cite{daiapprox}. 
We define the approximation operator for $\mathbb{S}^q$ by
\be\label{eq:sphapproxop}
\sigma_n(f)(x)\coloneqq\sigma_n(\mathbb{S}^q,f)(x)= \int_{\mathbb{S}^q}\Phi_{n,q}(x\cdot u)f(u)d\mu^*_q(u).
\ee
With this setup, we now review some bounds on how well $\sigma_n(f)$ approximates $f$ (cf. \cite{mhaskarsphere}).
\begin{proposition}\label{prop:sphapprox} Let $n\geq 1$.\newline
{\rm (a)} For  all $P\in \Pi_{n/2}^q$, we have $\sigma_n(P)=P$.\\
{\rm (b)}
For any $f\in C(\mathbb{S}^q)$, we have
\begin{equation}\label{eq:sphgoodapprox}
    E_n(f)\leq \norm{f-\sigma_n(f)}_{\mathbb{S}^q}\lesssim E_{n/2}(f).
\end{equation}
In particular, if $\gamma>0$ then $f\in W_\gamma(\mathbb{S}^q)$ if and only if
\be
\norm{f-\sigma_n(f)}_{\mathbb{S}^q}\ls \norm{f}_{W_\gamma(\mathbb{S}^q)}n^{-\gamma}.
\ee
\end{proposition}

\begin{rem}\label{rem:sphapprox}
Part (a) is known as a \textit{reproduction} property, which shows that polynomials up to degree $<n/2$ are unchanged when passed through the  operator $\sigma_n$. Part (b) demonstrates that $\sigma_n$ yields what we term a \textit{good approximation}, where its approximation error is no more than some constant multiple of the degree of approximation. Part (c)  not only gives the approximation bounds in terms of the smoothness parameter $\gamma$, but shows also that the rate of decrease of the approximation error obtained by $\sigma_n(f)$ \textbf{determines} the smoothness $\gamma$.\qed
\end{rem}

\subsection{Approximation on equators}
\label{subsec:eqapprox}

Let $SO(Q+1)$ denote group of all unitary $(Q+1)\times (Q+1)$ matrices with determinant equal to $1$. 
A \emph{$q$-dimensional equator} of $\mathbb{S}^Q$ is a set of the form $\mathbb{Y}=\{\mathcal{R}u:u\in\mathbb{S}^q\}$ for some $\mathcal{R}\in SO(Q+1)$.
 The goal in the remainder of this section is to give approximation results for equators. 
 
 Since there exist infinite options for $\mathcal{R}\in \operatorname{SO}(Q+1)$ to generate the set $\mathbb{Y}$, we first give a definition of degree of approximation in terms of spherical polynomials that is invariant to the choice of $\mathcal{R}$.

Fix $\mathbb{Y}$ to be a given $q$-dimensional equator of $\mathbb{S}^Q$ and let $\mathcal{R},\mathcal{S}\in \operatorname{SO}(Q+1)$ mapping $\mathbb{S}^q$ to $\mathbb{Y}$. Observe that if $P\in \Pi_n^q$, then $P(\mathcal{R}^T\mathcal{S}\circ)\in \Pi_n^q$ and vice versa. As a result, the functions $F_\mathcal{R}=f(\mathcal{R}\circ)$ and $F_\mathcal{S}=f(\mathcal{S}\circ)$ defined on $\mathbb{S}^q$ satisfy
\be
E_n(F_\mathcal{R})=E_n(F_\mathcal{S}).
\ee
Since the degree of approximation in this context is invariant to the choice of $\mathcal{R}\in \operatorname{SO}(Q+1)$, we may simply choose any such matrix that maps $\mathbb{S}^q$ to $\mathbb{Y}$, drop the subscript $\mathcal{R}$ from $F_\mathcal{R}$, and define
\be
E_n(\mathbb{Y},f)\coloneqq E_n(F).
\ee
This allows us to define the space $W_\gamma(\mathbb{Y})$ as the class of all $f\in C(\mathbb{Y})$ such that
\begin{equation}\label{eq:spheresmooth}
||f||_{W_\gamma(\mathbb{Y})}\coloneqq ||f||_{\mathbb{Y}}+\sup_{n\geq 0} 2^{n\gamma}E_{2^n}(\mathbb{Y},f)<\infty.
\end{equation}
We can also define the approximation operator on the set $\mathbb{Y}$ as
\begin{equation}\label{eq:sphapproxop2}
\sigma_n(\mathbb{Y},f)(x)\coloneqq \int_\mathbb{Y} \Phi_{n,q}(x\cdot y)f(y)d\mu^*_\mathbb{Y}(y),
\end{equation}
where $\mu^*_\mathbb{Y}(y)$ is the probability volume measure on $\mathbb{Y}$. Let $F_\mathcal{R}\in C(\mathbb{S}^q)$ satisfy $F_\mathcal{R}=f\circ \mathcal{R}$. We observe that
\begin{equation}
\begin{aligned}
\sigma_n(\mathbb{Y},f)(x)=&\int_{\mathbb{S}^q}\Phi_{n,q}(x\cdot \mathcal{R}u)f(\mathcal{R}u)d\mu_q^*(u)\\
=&\int_{\mathbb{S}^q}\Phi_{n,q}(\mathcal{R}^Tx\cdot u)F_\mathcal{R}(u)d\mu^*_q(u)\\
=&\sigma_n(\mathbb{S}^q,F_\mathcal{R})(\mathcal{R}^Tx).
\end{aligned}
\end{equation}

We now give an analogue of Proposition~\ref{prop:sphapprox} for approximation on equators.

\begin{theorem}\label{theo:subthm}
Let $f\in C(\mathbb{Y})$. \\
{\rm (a)} We have
\begin{equation}\label{eq:subthm1}
    E_n(\mathbb{Y},f)\leq\norm{\sigma_n(\mathbb{Y},f)-f}_\mathbb{Y}\lesssim E_{n/2}(\mathbb{Y},f).
\end{equation}
{\rm (b)} If $\gamma>0$, then $f\in W_\gamma(\mathbb{Y})$ if and only if
\begin{equation}\label{eq:subthm2}
\norm{\sigma_n(\mathbb{Y},f)-f}_\mathbb{Y}\lesssim n^{-\gamma}\norm{f}_{W_\gamma(\mathbb{Y})}.
\end{equation}

\end{theorem}

\begin{proof} Let $F(\circ)=f(\mathcal{R}\circ)$ for some $\mathcal{R}\in \operatorname{SO}(Q+1)$ with $\mathbb{Y}=\{\mathcal{R}u:u\in\mathbb{S}^q\}$. To see \eqref{eq:subthm1}, we check using Proposition~\ref{prop:sphapprox} that
\begin{equation}
\norm{\sigma_n(\mathbb{Y},f)-f}_\mathbb{Y}=\norm{\sigma_n(\mathbb{S}^q,F)(\mathcal{R}^T\circ)-F(\mathcal{R}^T\circ)}_\mathbb{Y}=\norm{\sigma_n(\mathbb{S}^q,F)-F}_{\mathbb{S}^q}\lesssim E_{n/2}(F)=E_{n/2}(\mathbb{Y},f).
\end{equation}
Additionally, $E_n(\mathbb{Y},f)\leq \norm{\sigma_n(\mathbb{Y},f)-f}_\mathbb{Y}$ since $\sigma_n(\mathbb{Y},f)=\sigma_n(\mathbb{S}^q,F)(\mathcal{R}^Tx)\in \Pi_n^q$.
Part (b) can be seen from part (a) and the definitions.
\end{proof}

\section{Function approximation on manifolds}
\label{sec:manifoldapprox}

In this section, we develop the notion of \textit{smoothness} for the target function defined on a manifold, and state our main theorem: Theorem~\ref{theo:mainthm}. For a brief introduction to manifolds and some results we will be using in this paper, see Appendix~\ref{sec:manifoldintro}.

Let $Q\geq q\geq 1$ be integers and $\mathbb{X}$ be a $q$-dimensional, compact, connected, submanifold of $\mathbb{S}^Q$ without boundary. 
Let $\rho$ denote the geodesic distance and $\mu^*$ be the normalized volume measure (that is, $\mu^*(\mathbb{X})=1$). 
For any $x\in \mathbb{X}$, observe that the tangent space $\mathbb{T}_x(\mathbb{X})$ is a $q$-dimensional vector space tangent to $\mathbb{S}^Q$. We define $\mathbb{S}_x=\mathbb{S}_x(\mathbb{X})$ to be the $q$-dimensional equator of $\mathbb{S}^Q$ passing through $x$ whose own tangent space at $x$ is also $\mathbb{T}_x(\mathbb{X})$. 
 As an important note, $\mathbb{S}_x$ is also a $q$-dimensional compact manifold.

In this paper we will consider many spaces, and need to define balls on each of these spaces, which we list in Table~\ref{tab:balldef} below.

\begin{table}[ht]

\begin{center}
\begin{tabular}{|c|c|c|}
\hline
Space & Description & Definition\\
\hline
Ambient space & Euclidean ball & $B_{Q+1}(x,r)=\{y\in \mathbb{R}^{Q+1}:\norm{x-y}_2\leq r\}$\\
\hline Ambient sphere & Spherical cap & $S^Q(x,r)=\{y\in \mathbb{S}^Q:\arccos(x\cdot y) \le r\}$\\
\hline
Tangent space & Tangent ball & $B_{\mathbb{T}}(x,r)=\{y\in \mathbb{T}_x(\mathbb{X}): ||x-y||_{2}\le r\}$\\
\hline
Tangent sphere & Tangent cap & $\mathbb{S}_x(r)=\{y\in \mathbb{S}_x:\arccos(x\cdot y)\le r\}$ \\
\hline
Manifold & Geodesic ball & $\mathbb{B}(x,r)=\{y\in \mathbb{X}:\rho(x,y)\le  r\}$\\
\hline
\end{tabular}\end{center}
\caption{Definition and description of balls in different spaces.}
\label{tab:balldef}
\end{table}

We also need to define the smoothness classes we will be considering for functions on $\mathbb{X}$. 
Let $C(\mathbb{X})$ denote the space of all continuous functions on $\mathbb{X}$, and $C^\infty(\mathbb{X})\subset C(\mathbb{X})$ denote the space of all infinitely differentiable functions on $\mathbb{X}$. 
Let $\overline{\varepsilon}_x$ be the exponential map at $x$ for $\mathbb{S}_x$ and $\varepsilon_x$ be the exponential map at $x$ for $\mathbb{X}$.
Since both $\mathbb{X}$ and $\mathbb{S}_x$ are compact, we have some $\iota_1,\iota_2$ such that $\varepsilon_x,\overline{\varepsilon}_x$ are defined on $B_{\mathbb{T}}(x,\iota_1),B_{\mathbb{T}}(x,\iota_2)$ respectively for any $x$. 
We write $\iota^*=\min\{1, \iota_1,\iota_2\}$ and define $\eta_x :\mathbb{S}_x(\iota^*)\to \XX$ by 
 $\eta_x: \varepsilon_x \circ\overline{\varepsilon}_x^{-1}$.
 Thus,
\be\label{eq:rho_sphdist}
\rho(x,\eta_x(y))=\arccos(x\cdot y), \qquad x\in \XX, \ y\in \mathbb{S}_x(\iota^*).
\ee

\begin{definition}\label{def:manifold_smoothness}
We say that $f\in C(\mathbb{X})$ is \textbf{$\gamma$-smooth} for some $\gamma>0$, or also that $f\in W_\gamma(\mathbb{X})$, if for every $x\in \mathbb{X}$ and $\phi\in C^\infty(\mathbb{X})$ supported on $\mathbb{B}(x,\iota^*)$, the function $F_{x,\phi}:\mathbb{S}_x\to \mathbb{R}$ defined by $F_{x,\phi}\coloneqq f(\eta_x(u))\phi(\eta_x(u))$ belongs to $W_\gamma(\mathbb{S}_x)$ as outlined in Section~\ref{subsec:sphapprox} (in particular, Equation~\eqref{eq:spheresmooth}). 
We also define
\be\label{eq:manifold_sobnorm}
\norm{f}_{W_\gamma(\mathbb{X})}\coloneqq \sup_{x\in\mathbb{X},\norm{\phi}_{W_\gamma(\mathbb{S}_x)}\leq 1} \norm{F_{x,\phi}}_{W_\gamma(\mathbb{S}_x)}.
\ee
\end{definition}

Our main theorem, describing the approximation of $ff_0$ (the target function weighted by the density of data points) by the operator defined in \eqref{eq:proto_const}, is the following.
We note that approximation of $ff_0$ includes local approximation on $\XX$ in the sense that when the training data is sampled only from a subset of $\XX$, this fact can be encoded by $f_0$ being supported on this subset.

\begin{theorem}\label{theo:mainthm}
We assume that
\begin{equation}\label{eq:ballmeasure}
    \sup_{x\in \mathbb{X},r>0}\frac{\mu^*(\mathbb{B}(x,r))}{r^q}\ls 1.
\end{equation}
Let $\mathcal{D}=\{(y_j,z_j)\}_{j=1}^M$ be a random sample  from a joint distribution $\tau$. 
We assume that the  marginal distribution of $\tau$ restricted to $\mathbb{X}$ is absolutely continuous with respect to $\mu^*$ with density $f_0$, and that the random variable $z$ has a bounded range. We say $z\in [-\norm{z},\norm{z}]$.
Let
\begin{equation}\label{eq:fdef}
    f(y)\coloneqq \mathbb{E}_\tau(z|y),
\end{equation}
and
\begin{equation}\label{eq:approximation}
    F_{n}(\mathcal{D};x)\coloneqq \frac{1}{M}\sum_{j=1}^M z_j\Phi_{n,q}(x\cdot y_j),\qquad x\in \mathbb{S}^Q,
\end{equation}
where $\Phi_{n,q}$ is defined in \eqref{eq:kernel}.

Let $0<\gamma<2$ and $ff_0\in W_\gamma(\mathbb{X})$.
 Then for every $n\geq 1$, $0<\delta<1/2$ and 
 \be\label{eq:Mncond}
  M\gtrsim n^{q+2\gamma}\log(n/\delta),
  \ee
   we have with $\tau$-probability $\geq 1-\delta$:
\begin{equation}\label{eq:approxest}
    \norm{F_n(\mathcal{D};\circ)-ff_0}_\mathbb{X}\lesssim \frac{\sqrt{\norm{f_0}_{\mathbb{X}}}\norm{z}+\norm{ff_0}_{W_\gamma(\mathbb{X})}}{n^\gamma}.
\end{equation}
Equivalently, for integer $M\ge 2$ and $n$ satisfying \eqref{eq:Mncond}, we have with $\tau$-probability $\geq 1-\delta$:
\be\label{eq:sample_approxest}
 \norm{F_n(\mathcal{D};\circ)-ff_0}_\mathbb{X}\lesssim \left\{\sqrt{\norm{f_0}_{\mathbb{X}}}\norm{z}+\norm{ff_0}_{W_\gamma(\mathbb{X})}\right\}\left(\frac{\log(M/\delta^{q+2\gamma})}{M}\right)^{\gamma/(q+2\gamma)}.
\ee
\end{theorem}

We discuss two corollaries of this theorem, which demonstrate how the theorem can be used for both estimation of the probability density $f_0$ and the approximation of the function $f$ in the case when the training data is sampled from the volume measure on $\XX$.

The first corollary is a result on function approximation in the case when the marginal distribution of $y$ is $\mu^*$; i.e., $f_0\equiv 1$.

\begin{corollary}\label{cor:maincor}
Assume the setup of Theorem~\ref{theo:mainthm}. Suppose also that the marginal distribution of $\tau$ restricted to $\mathbb{X}$ is uniform. Then we have with $\tau$-probability $\geq 1-\delta$:
\begin{equation}\label{eq:fnapprox}
    \norm{F_n(\mathcal{D};\circ)-f}_\mathbb{X}\lesssim \frac{\norm{z}+\norm{f}_{W_\gamma(\mathbb{X})}}{n^\gamma}.
\end{equation}
\end{corollary}

The second corollary is obtained by setting $f\equiv 1$, to point out that our theorem gives a method for density estimation. 
In practice, one may not have knowledge of $f_0$ (or even the manifold $\mathbb{X}$). So, the following corollary can be applied to estimate this critical quantity. We use this fact in our numerical examples in Section~\ref{sec:numerical}.

Typically, a positive kernel is used for the problem of density estimation in order to ensure that the approximation is also a positive measure.
It is well known in approximation theory that this results in a saturation for the rate of convergence.
Our method does not use positive kernels, and does not suffer from such saturation.

\begin{corollary}\label{cor:densityest}
Assume the setup of Theorem~\ref{theo:mainthm}. Then
we have with $\tau$-probability $\geq 1-\delta$:
\be\label{eq:densityest}
\norm{\ \ \left|\frac{1}{M}\sum_{j=1}^M\Phi_{n,q}(\circ\cdot y_j)\right|-f_0}_\mathbb{X}\lesssim \frac{\norm{f_0}_{W_\gamma(\mathbb{X})}}{n^\gamma}.
\ee
\end{corollary}

\section{Numerical examples}
\label{sec:numerical}

In this section, we illustrate our theory with some numerical experiments.
In Section~\ref{subsec:expiecewise}, we consider the approximation of a piecewise differentiable function, and demonstrate how the localization of the kernel leads to a determination of the locations of the singularities. 
The example in Section~\ref{subsec:spencerdata} is motivated by magnetic resonance relaxomety. 
Since it is relevant to our method for practical applications, we have included some discussion and results about how $q$ effects the approximation in this example. 
The example in Section~\ref{subsec:darcyflow} illustrates how our method can be used for inverse problems in the realm of differential equations.
In all the examples, we will examine how the accuracy of the approximation depends on the maximal  degree $n$ of the polynomial, the number $M$ of samples, and the level of noise.

\subsection{Piecewise differentiable function}
\label{subsec:expiecewise}

In this example only we define the function to be approximated as
\be
f(\theta)\coloneqq 1+\abs{\cos\theta}^{1/2}\sin(\cos\theta+\sin\theta)/2,
\ee 
defined on the ellipse
\be
E=\{(3\cos\theta,6\sin\theta): \theta\in (-\pi,\pi]\}.
\ee
We project $E$ to the sphere $\mathbb{S}^2$ using an inverse stereographic projection defined by
\be
\mathbf{P}(\vec{x})=\frac{(\vec{x},1)}{\norm{(\vec{x},1)}_2},
\ee 
and call $\mathbb{X}=\mathbf{P}(E)$. Each $\vec{x}\in\mathbb{X}$ is associated with the value $\theta_\vec{x}$ satisfying $\vec{x}=\mathbf{P}\big((3\cos\theta_\vec{x},6\sin\theta_\vec{x})\big)$, so that $f(\vec{x})\coloneqq f(\theta_\vec{x})$ is a continuous function on $\mathbb{X}$.

We generate our data points by taking $\vec{y}_j=\mathbf{P}\big((3\cos\theta_j,6\sin\theta_j)\big)$, where $\theta_j$ are each sampled uniformly at random from $(-\pi,\pi]$. We then define $z_j=f(\vec{y}_j)+\epsilon_j$ where $\epsilon_j$ are sampled from some mean-zero normal noise. 
Our data set is thus $\mathcal{D}\coloneqq\{(\vec{y}_j,f(\vec{y}_j)+\epsilon_j)\}_{j=1}^M$.
We will measure the magnitude of noise using the signal-to-noise ratio (SNR), defined by 
\be\label{eq:snr}
20\log_{10}\Big(\norm{\big(f(\vec{y}_1),\dots,f(\vec{y}_M)\big)}_2\Big/\norm{\big(\epsilon_1,\dots,\epsilon_M\big)}_2\Big).
\ee  
Since $f_0\neq 1$ in this case, we could calculate $f_0$ from the projection, or we may estimate it using Corollary~\ref{cor:densityest}. That is,
\be
f_0(\vec{x})\approx\frac{1}{M}\sum_{j=1}^M\Phi_{n,1}(\vec{x}\cdot \vec{y}_j).
\ee
This option may be desirable in cases where $f_0$ is not feasible to compute (i.e. if the underlying domain of the data is unknown or irregularly shaped). Our approximation is thus:
\be
F_n(\mathcal{D};\vec{x})= \sum_{j=1}^M (f(\vec{y}_j)+\epsilon_j)\Phi_{n,1}(\vec{x}\cdot \vec{y}_j)\left/\left(\sum_{j=1}^M \Phi_{n,1}(\vec{x}\cdot \vec{y}_j)\right).\right.
\ee

\begin{figure}[!ht]
\centering
    \includegraphics[width=.4\textwidth]{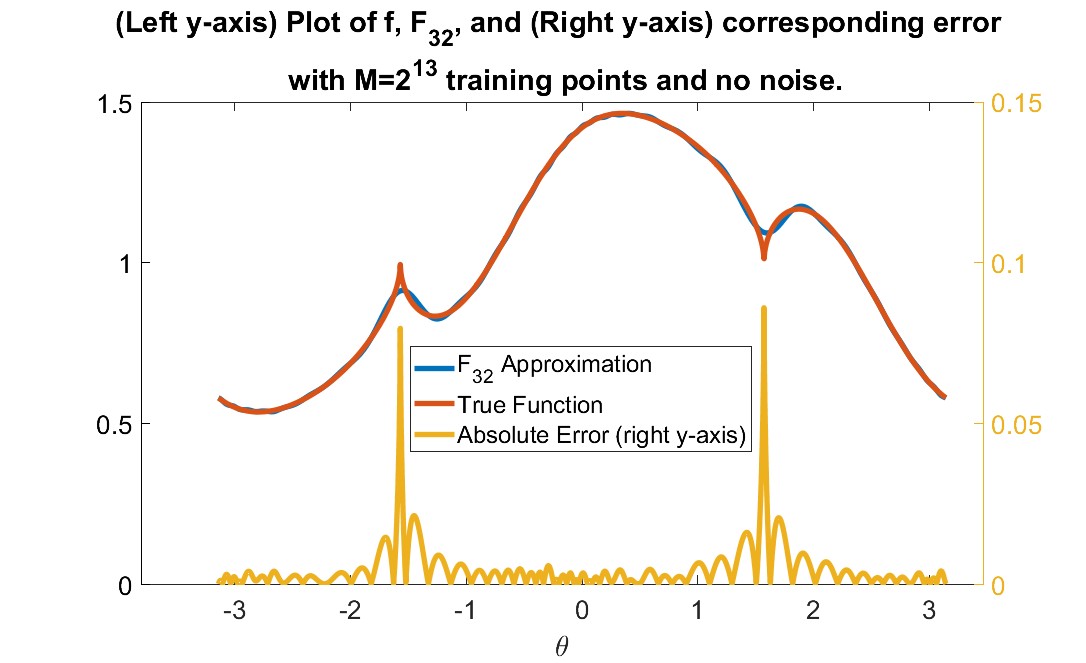}\\
    \caption{Left y-axis: Plot of the true function $f$ compared with $F_{32}$ constructed by $2^{13}$ noiseless training points. Right y-axis: Plot of $\abs{f-F_{32}}$.}
    \label{fig:errorplot}
\end{figure}

\begin{figure*}[!ht]
\centering
\begin{tabular}{ccc}
    \includegraphics[width=.3\textwidth]{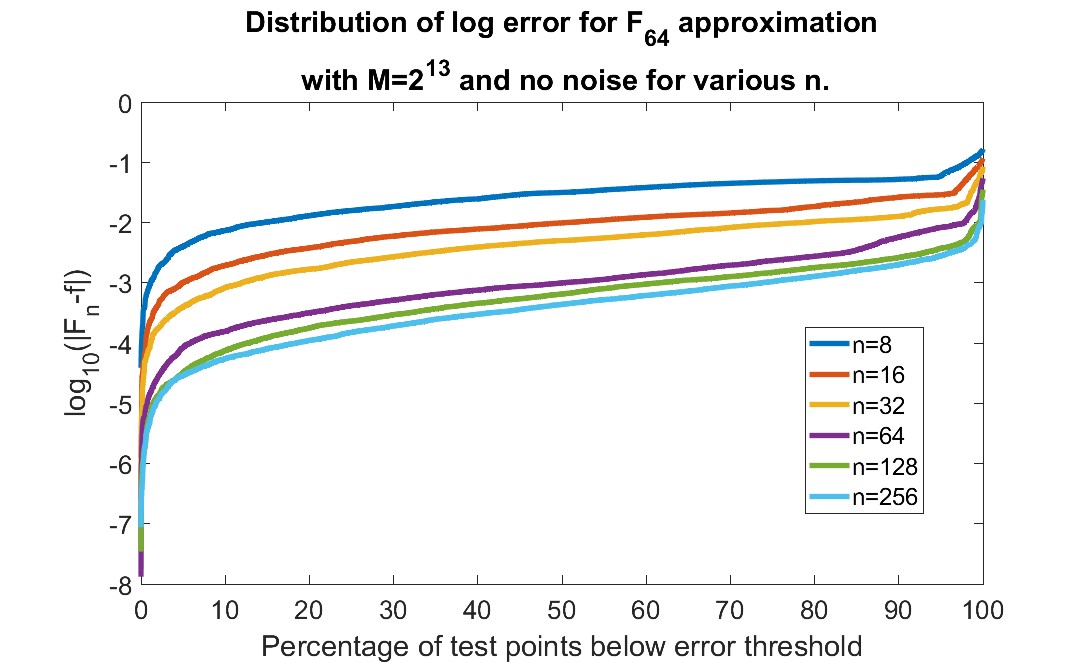}
&
    \includegraphics[width=.3\textwidth]{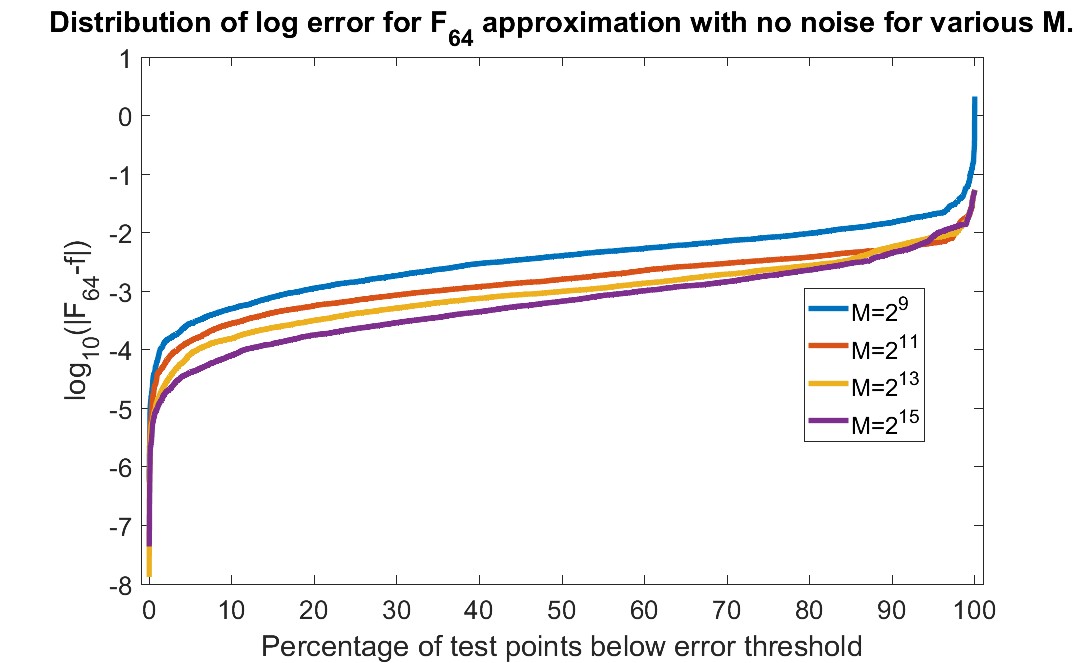}
&
    \includegraphics[width=.3\textwidth]{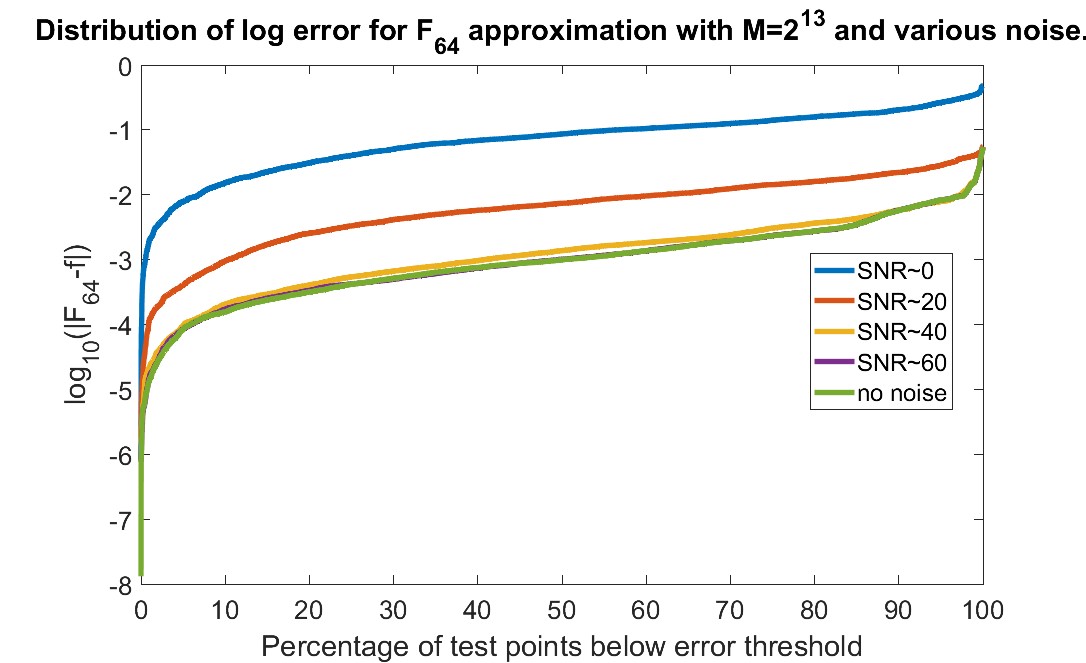}
\end{tabular}
    \caption{(Left) Percent point plot of log absolute error for various $n$ with $M=2^{13}$ training points and no noise. (Center) Percent point plot of log absolute error for various choices of $M$ with no noise. (Right) Percent point plot of log absolute error for various noise levels with $M=2^{13}$ training points.}
        \label{fig:3plot}
\end{figure*}

Figure~\ref{fig:errorplot} shows a plot of the true function and our approximation on the left y-axis and the absolute error on the right y-axis. The plot demonstrates that the approximation achieves much lower error than the uniform error bound at points where the function is relatively smooth, and only spikes locally at the singularities of the function ($\theta=\pm \pi/2$). Figure~\ref{fig:3plot} displays three percent point plots illustrating how the distribution of $\log_{10}|F_n-f|$ behaves for various choices of $n,M,\epsilon$. Each point $(x,y)$ on a curve indicates that $x\%$ of test points were approximated by our method with absolute error below $10^{y}$ for the $n$, $M$, and $\epsilon$ value associated with the curve.
The first graph shows the trend for various $n$ values.  As we increase $n$, we see consistent drop in log error.
The second graph shows various values of $M$. We again see a decrease in the overall log error as $M$ is increased.
The third graph shows how the log error decreases as the noise decreases. We can see that the approximation is much worse for low SNR values, but nearly indistinguishable from the noiseless case when the SNR is above 60.

\subsection{Parameter estimation in bi-exponential sums}
\label{subsec:spencerdata}

This example is motivated by magnetic resonance relaxometry, in which the proton nuclei of water are first
excited with radio frequency pulses and then exhibit an exponentially decaying electromagnetic signal.
When one may assume the presence of two water compartments undergoing slow exchange, with signal corrupted by additive Gaussian noise, the
 signal is modeled typically as a bi-exponential decay function  \eqref{eq:ex2eq} (cf. \cite{Spencer-biexponential}):
$$
 F(t)=c_1\exp(-t/T_{2,1})+c_2\exp(-t/T_{2,2})+E(t),
 $$
 where $E$ is the noise, $T_{2,1}, T_{2,2}>0$, and the time $t$ is typically sampled at equal intervals.
The problem is to determine $c_1, c_2, T_{2,1}, T_{2,2}$. 
The problem appears also in many other medical applications, such as intravoxel incoherent motion studies in magnetic resonance.
 An accessible survey of these applications is given in \cite{Istratov1999}.
 
 Writing $t=j\delta$, $\lambda_1=\delta/T_{2,1}$, $\lambda_2=\delta/T_{2,2}$, we may reformulate the data as
\be\label{eq:ex2eq}
f(j)\coloneqq f(\lambda_1,\lambda_2,j)=c_1 e^{-\lambda_1j}+c_2 e^{-\lambda_2j}+\epsilon(j),
\ee 
where $\epsilon(j)$ are samples of mean-zero normal noise.

In this example, suggested by Dr. Spencer at the National Institute of Aging (NIH), we  consider the case where $c_1=.7,c_2=.3$ and use our method to  
 determine the values $\lambda_1,\lambda_2$, given data of the form
\be\label{eq:ex2eq2}
\tilde{\mathbf{y}}(\lambda_1,\lambda_2)\coloneqq(f(1),f(2),\dots,f(100)).
\ee
 We ``train" our approximation process with $M$ samples of $(\lambda_1,\lambda_2)\in [.1,.7]\times [1.1,1.7]$ chosen uniformly at random and then plugging those values into~\eqref{eq:ex2eq} to generate vectors of the form shown in~\eqref{eq:ex2eq2}. 
 The dimension of the input data is $Q=100$, however (in the noiseless case) the data lies on a $q=2$ dimensional manifold, so we will use $\Phi_{n,2}$ to generate our approximations. 
 
 We note that our method is agnostic to the particular model \eqref{eq:ex2eq2} used to generate the data. 
 We treat $\lambda_1, \lambda_2$ as functions of $\tilde{\mathbf{y}}$ without a prior knowledge of this function.
 In the noisy case, this problem does not perfectly fit the theory studied in this paper since the noise is applied to the input values $f(t)$ meaning we cannot assume they lie directly on an unknown manifold anymore. 
 Nevertheless, we can see some success with our method. 
We define the operators
\be
\mathbf{T}(\tilde{\mathbf{y}})=1000\tilde{\mathbf{y}}-(380,189,116,0,\dots,0),\qquad \mathbf{P}(\circ)=\frac{(\circ, 100)}{\norm{(\circ, 100)}_2}
\ee
and denote $\mathbf{y}=\mathbf{P}(\mathbf{T}(\tilde{\mathbf{y}}))$.
 In practice, the values used to define $\mathbf{T}$ and $\mathbf{P}$   need to be treated as hyperparameters of the model. In this example, we did not conduct a rigorous grid search.
 We use the same density estimation as in Section~\ref{subsec:expiecewise}:
\be
\operatorname{DE}\big(\x(\lambda_1,\lambda_2)\big)=\sum_{j=1}^M\Phi_{n,2}\big(\vec{x}(\lambda_1,\lambda_2)\cdot \mathbf{y}(\lambda_{1,j}, \lambda_{2,j})\big).
\ee
As a result, our approximation process looks like:
\be
\begin{bmatrix}\lambda_1\\
\lambda_2\end{bmatrix}\approx \vec{F}_n\big(\vec{x}(\lambda_1,\lambda_2)\big)=\sum_{j=1}^M\begin{bmatrix}\lambda_{1,j}\\\lambda_{2,j}\end{bmatrix}\Phi_{n,2}\big(\x(\lambda_1,\lambda_2)\cdot \mathbf{y}(\lambda_{1,j}, \lambda_{2,j})\big)\Big/\operatorname{DE}\big(\x(\lambda_1,\lambda_2)\big).
\ee 

Similar to Example~\ref{subsec:expiecewise}, we will include figures showing how the results are effected as $n,M,\epsilon$ are adjusted. We measure noise using the signal-to-noise ratio (SNR) defined by
\be
20\log_{10}\left(\norm{\tilde{\vec{y}}}_2\Big/\norm{\big(\epsilon(1),\dots,\epsilon(100)\big)}_2\right).
\ee
Unlike Example~\ref{subsec:expiecewise}, we will now be considering percent approximation error instead of uniform error as it is more relevant in this problem. We define the \textit{combined error} to be
\be
\sum_{j=1}^2 \frac{\abs{\lambda_{j,\text{true}}-\lambda_{j,\text{approx}}}}{\lambda_{j,\text{true}}}.
\ee

\begin{figure*}[!ht]
\centering
\begin{tabular}{ccc}
    \includegraphics[width=.3\textwidth]{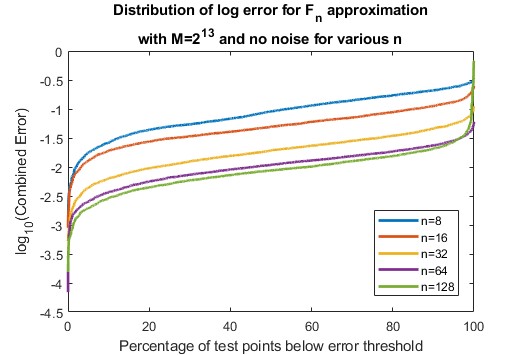}
&
    \includegraphics[width=.3\textwidth]{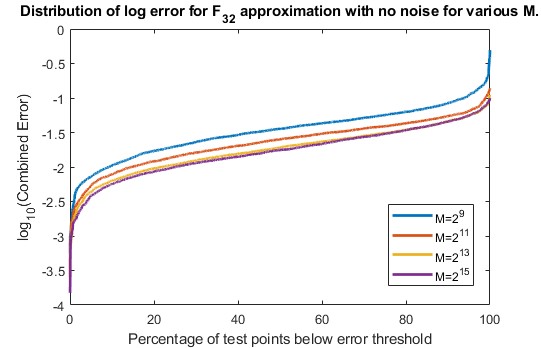}
&
    \includegraphics[width=.3\textwidth]{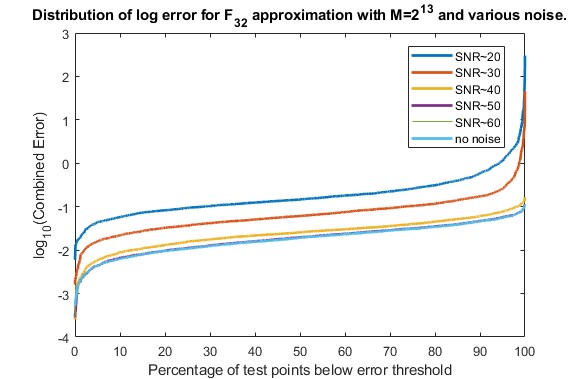}
\end{tabular}
    \caption{(Left) Percent point plot of log combined error for various $n$ with $M=2^{13}$ training points, and no noise. (Center) Percent point plot of log combined error for fixed $n=32$, various choices of $M$, and no noise. (Right) Percent point plot of log combined error for fixed $n=32$, fixed $M=2^{13}$ training points, and various noise levels.}
        \label{fig:ex23plot}
\end{figure*}

\begin{figure*}[!ht]
\centering
    \includegraphics[width=.45\textwidth]{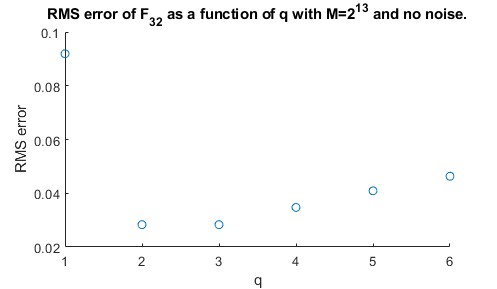}
    \caption{Plot of RMS error for approximation by $F_{32}$ for various $q$ values with $M=2^{13}$ and no noise.}
        \label{fig:ex2q}
\end{figure*}

Figure~\ref{fig:ex23plot} depicts three percent point plots showing the distribution of sorted $\log_{10}$(Combined Error) points for various $n,M,\epsilon$. Each point $(x,y)$ on a curve indicates that $x\%$ of test points were approximated by our method with combined error below $10^{y}$ for the $n$, $M$, and $\epsilon$ associated with the curve.
In the first graph, we see the distribution of for various choices of $n$. 
As $n$ increases, the overall log error decreases. 
An interesting phenomenon occurring in this figure is with the $n=128$ case where the uniform error is actually higher than the $n=64$ case. 
This is likely due to the fact  that overfitting can occur if $n$ gets too large relative to a fixed $M$. 
The second graph illustrates how the approximation improves as $M$ is increased. 
As expected, we see log error decay as we include more and more training points. In the third graph, we see that the approximation improves up to a limit as the noise decreases. 
There is very little noticeable difference between the noiseless case and any case where SNR$>50$.

Another question that may arise when utilizing our method on various data is what value of $q$ to use. While the theory predicts that $q$ should be associated with the intrinsic dimension of the manifold underlying the data, in practice this can only be estimated and so $q$ should be treated as a hyperparameter. In Figure~\ref{fig:ex2q}, we explore how changing $q$ effects the approximation in this example. In this case, the intrinsic dimension is $2$, and when $q=2,3$ the approximation does well. If $q$ is chosen too high or two low, the approximation yields a greater error.

\subsection{Darcy flow problem}
\label{subsec:darcyflow}

In this section we will look at a numerical example from the realm of PDE inverse problems. Steady-state Darcy flow is given by the following PDE (see for example,  \cite[Eq. (4.7)]{raonic2024convolutional}):
\be
-\nabla\cdot(a\nabla y)=f,
\ee
defined on a domain $D$ with the property that $y|_{\partial D}=0$. The problem is to predict the \textit{diffusion coefficient} $a$ and \textit{forcing term} $f$ given some noisy samples of $y$ on $D$. In this paper we consider a 1-dimensional version and suppose that $a=e^{-st}$ and $f=pe^{-st}$ for some $p,s$. We take noisy samples of $y(p,s;\circ)=y$ satisfying the following boundary value problem:
\be\label{eq:ex3difeq}
-\frac{d}{dt}(e^{-st}y'(t))=pe^{-st},\qquad y(1)=0, y(0)=1.
\ee
In this sample, we take a similar approach to that of Example~\ref{subsec:spencerdata} by ``training'' our model with a data set of the form $\{\vec{y}_j,(p_j,s_j)\}_{j=1}^M$, where $(p_j,s_j)\in [.1,.25]\times [1.5,2.5]$ are sampled uniformly at random for each $j$. Letting $y_j$ denote the $y$ satisfying \eqref{eq:ex3difeq} with $p=p_j,s=s_j$, then $\vec{y}_j=\mathbf{P}(y_j(t_1),y_j(t_2),\dots,y_j(t_{100}))$, where $t_1,t_2,\dots,t_{100}$ are sampled uniformly from $[0,1]$ and $\mathbf{P}$ is the projection to the sphere. In this example, the projection first consists of finding the center $C$ and maximum spread over a single feature $r$ of the data. That is,
\be
\ba
C=&\left(\max_{j}y_j(t_1)+\min_j y_j(t_1),\dots,\max_j y_j(t_{100})+\min_j y_j(t_{100})\right)\Big/2,\\ r=&\max_{t_i}\left(\max_{j}y_j(t_1)-\min_j y_j(t_1),\dots,\max_j y_j(t_{100})-\min_j y_j(t_{100})\right).
\ea
\ee
Then, we define
\be
\mathbf{P}(\circ)=\frac{(\circ-C,r)}{\norm{(\circ-C,r)}_2}.
\ee
Our approximation process then looks like:
\be
\begin{bmatrix}p\\
s\end{bmatrix}\approx \vec{F}_n\big(\vec{y}\big)=\sum_{j=1}^M\begin{bmatrix}p_j\\s_j\end{bmatrix}\Phi_{n,2}(\vec{y}\cdot \mathbf{y}_j)\Big/\operatorname{DE}(\vec{y}),
\ee 
where
\be
\operatorname{DE}(\vec{y})=\sum_{j=1}^M \Phi_{n,2}(\vec{y}\cdot \vec{y}_j).
\ee
Also similar to Example~\ref{subsec:spencerdata}, we use the same notion of SNR and evaluate the success of our model using a \textit{combined error}, now defined to be
\be
\left(\abs{\frac{p_{\text{true}}-p_{\text{approx}}}{p_{\text{true}}}}+\abs{\frac{s_{\text{true}}-s_{\text{approx}}}{s_{\text{true}}}}\right).
\ee

\begin{figure*}[!ht]
\centering
\begin{tabular}{ccc}
    \includegraphics[width=.3\textwidth]{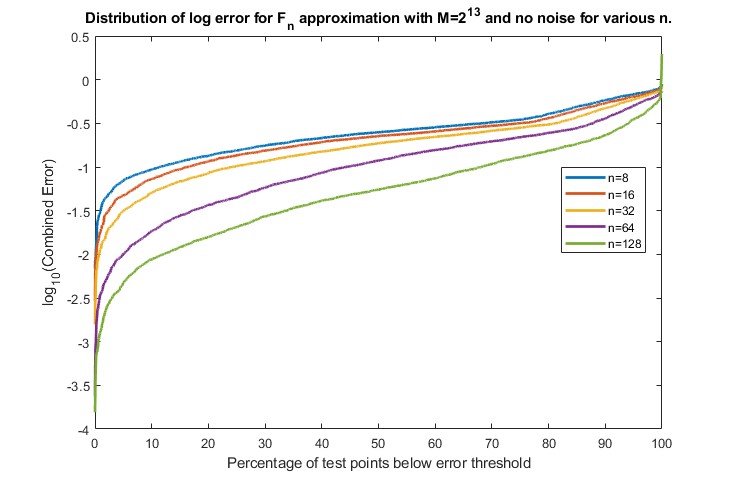}
&
    \includegraphics[width=.3\textwidth]{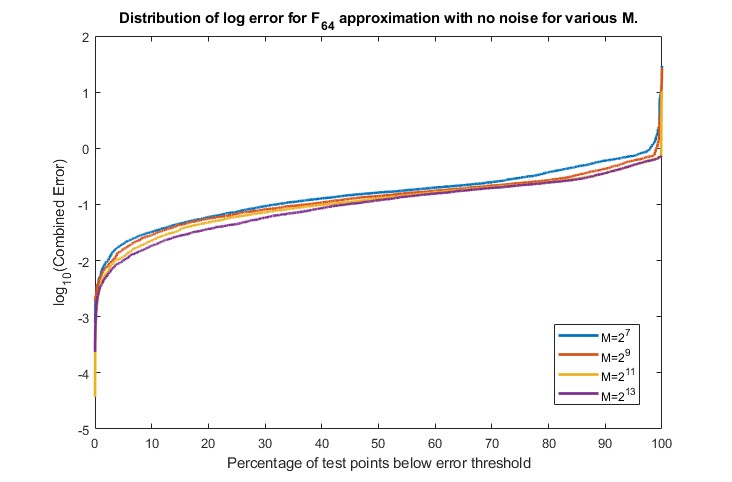}
&
    \includegraphics[width=.3\textwidth]{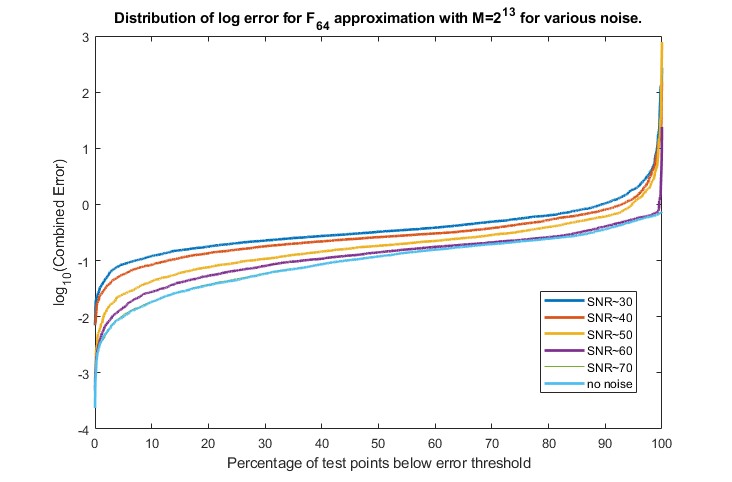}
\end{tabular}
\caption{(Left) Percent point plot of log combined error for various $n$, fixed $M=2^{13}$, and no noise. (Center) Percent point plot of log combined error for fixed $n=64$, various $M$, and no noise. (Right) Percent point plot of log combined error for fixed $n=64$, fixed $M=2^{13}$, and various noise levels.}
        \label{fig:ex3plot}
\end{figure*}

In Figure~\ref{fig:ex3plot}, we provide some percent point plots from using our method on this data. Each point $(x,y)$ on a curve indicates that $x\%$ of test points were approximated by our method with combined error below $10^{y}$ for the $n$, $M$, and $\epsilon$ associated with the curve. We see in the left-most plot that as we increase $n$, the error tends to decrease. In contrast to previous examples, the middle plot does not show much improvement by increasing M. This may be an indication of the fact that we have chosen a tight parameter space in this example (as compared to \ref{subsec:spencerdata}) and not many samples are needed to sufficiently cover the space. On the right-most plot, we see a decrease in error with the decrease of noise as expected, with convergence appearing to occur around the SNR=70 mark, as indicated by the green and light-blue lines being so close together.

\section{Proofs}\label{sec:proofs}

The purpose of this section is to prove Theorem~\ref{theo:mainthm}. 

In Section~\ref{bhag:summabilityop}, we study the approximation properties of the integral reconstruction operator  defined in \eqref{eq:manifold_summabilityop} (Theorem~\ref{theo:manifold_summop}).
In Section~\ref{bhag:resultsproof}, we use this theorem with $ff_0$ in place of $f$, and  use the Bernstein concentration inequality (Proposition~\ref{prop:berncon}) to discretize the integral expression in \eqref{eq:manifold_summabilityop} and complete the proof of Theorem~\ref{theo:mainthm}.

\subsection{Integral reconstruction operator}\label{bhag:summabilityop}
In this section, we prove the following theorem which is an integral analogue of Theorem~\ref{theo:mainthm}.

\begin{theorem}\label{theo:manifold_summop}
Let $0<\gamma<2$, $f\in W_\gamma(\mathbb{X})$, $\sigma_n$ be as defined in \eqref{eq:manifold_summabilityop}.
 Then for $n\geq 1$, we have
\begin{equation}\label{eq:theobound}
    \norm{f-\sigma_n(\mathbb{X},f)}_\mathbb{\mathbb{X}}\lesssim n^{-\gamma}\norm{f}_{W_\gamma(\mathbb{X})}.
\end{equation}
\end{theorem}

In order to prove this theorem, we will use a covering of $\XX$ using balls of radius $\iota^*$, and a corresponding partition of unity. 
A key lemma to facilitate the details here is the following.

\begin{lemma}\label{lemma:manifold_keylemma}
Let  $x\in \mathbb{X}$. 
Let $g\in C(\mathbb{X})$ be supported on $\mathbb{B}(x,\iota^*)$. If $G(u)=g(\eta_x(u))$, $0<\gamma<2$. 
Then
\begin{equation}\label{eq:manifold_localest}
    \abs{\int_\mathbb{X} \Phi_{n,q}(x\cdot y)g(y)d\mu^*(y)-\int_{\mathbb{S}_x} \Phi_{n,q}(x\cdot u) G(u)d\mu_{\mathbb{S}_x}^*(u)}\lesssim n^{-\gamma}||g||_\mathbb{X}.
\end{equation}
\end{lemma}

If $\phi\in C^\infty(\XX)$ is supported on $\mathbb{B}(x,\iota^*)$, then we may apply this theorem with $g=f\phi$, thereby providing locally a lifting of the integral on $\XX$ to the tangent equator $\SS_x$ with the function corresponding to $g$ on this tangent equator.

Naturally, the first step in this proof is to show that the Lebesgue constant for the kernel $\Phi_{n,q}$ is bounded independently of $n$ (cf. \eqref{eq:sigmaall}). 
Moreover, one can even leverage the localization of the kernel to improve on this bound when the integral is taken away from the point $x$ (cf. \eqref{eq:sigmaaway}). These are both done in the following lemma.

\begin{lemma}\label{lemma:manifold_lebesgue_number}
Let $r>0$ and $n\geq 1/r$. If  $\Phi_{n,q}$ is given as in \eqref{eq:kernel} with $S>q$, then
\be\label{eq:sigmaaway}
\sup_{x\in \mathbb{X}}\int_{\mathbb{X}\setminus \mathbb{B}(x,r)}\abs{\Phi_{n,q}(x\cdot y)}d\mu^*(y)\lesssim \max(1,nr)^{q-S}.
\ee
Additionally,
\be\label{eq:sigmaall}
\sup_{x\in\mathbb{X}}\int_{\mathbb{X}}|\Phi_{n,q}(x\cdot y)|d\mu^*(y)\lesssim 1.
\ee
\end{lemma}

\begin{proof}\ 
Recall from Proposition~\ref{prop:taylorprop} that $\rho(x,y)\sim \arccos(x\cdot y)$, so \eqref{eq:sphkernloc} implies
\be\label{eq:manifoldloc}
\abs{\Phi_{n,q}(x\cdot y)}\lesssim \frac{n^q}{\max(1,n\arccos(x\cdot y))^S}\lesssim \frac{n^q}{\max(1,n\rho(x, y))^S}.
\ee
In this proof only, we fix $x\in \mathbb{X}$. 
Let $A_0=\mathbb{B}(x,r)$ and $A_k=\mathbb{B}(x,2^kr)\setminus \mathbb{B}(x,2^{k-1}r)$, $k\ge 1$.
Then $\mu^*(A_k)\ls 2^kqr^q$, and
  for any $y\in A_k$, $2^{k-1}r\leq \rho(x,y)\leq 2^{k}r$.
  
  First, let  $nr\geq 1$. 
 In view of \eqref{eq:ballmeasure}~and~\eqref{eq:manifoldloc}, it follows that
\be\ba\label{eq:lem01}
\int_{\mathbb{X}\setminus \mathbb{B}(x,r)}\abs{\Phi_{n,q}(x\cdot y)}d\mu^*(y)=&\sum_{k=1}^\infty \int_{A_k} \abs{\Phi_{n,q}(x\cdot y)}d\mu^*(y)
\lesssim \sum_{k=1}^\infty \frac{\mu^*(A_k)n^q}{(n2^{k-1}r)^S}\\
\lesssim&(nr)^{q-S}\sum_{k=0}^\infty 2^{k(q-S)}
\leq (nr)^{q-S}.
\ea\ee
Using this estimate with $r=1/n$ and the condition \eqref{eq:ballmeasure} on the measures of balls we see that
$$
\int_{\mathbb{X}}\abs{\Phi_{n,q}(x\cdot y)}d\mu^*(y)=\int_{A_0}\abs{\Phi_{n,q}(x\cdot y)}d\mu^*(y)+\int_{\mathbb{X}\setminus\mathbb{B}(x,r)}\abs{\Phi_{n,q}(x\cdot y)}d\mu^*(y)\lesssim 1+(nr)^{q-S}\sim 1.
$$
Since the choice of $x$ was arbitrary, we have proven \eqref{eq:sigmaall}. Then \eqref{eq:sigmaall}~and~\eqref{eq:lem01} combined give the bounds for \eqref{eq:sigmaaway}.
\end{proof}

Next, we prove Lemma~\ref{lemma:manifold_keylemma}.\\[1ex]

\noindent\textsc{Proof of  Lemma~\ref{lemma:manifold_keylemma}.}

Since $\gamma<2$, we may choose (for sufficiently large $n$)
\be\label{eq:deltaSchoice}
 \delta=n^{-(\gamma+q+1)/(q+3)}, \quad n\delta=n^{(2-\gamma)/(q+3)} >1, \quad S> \frac{2q+3\gamma}{2-\gamma}.
 \ee
We may assume further that $\delta<\iota^*$.
Then, by using \eqref{eq:sph_bernstein} and Proposition~\ref{prop:taylorprop}, we see that
\be\label{eq:phidiff}
\ba
\biggl|\Phi_{n,q}(x\cdot \eta_x(u))&-\Phi_{n,q}(x\cdot u)\biggr|\lesssim n^{q+1}\abs{\arccos(x\cdot \eta_x(u))-\arccos(x\cdot u)}\\
=&n^{q+1}\abs{\arccos(x\cdot \eta_x(u))-\rho(x, \eta_x(u))}
\lesssim n^{q+1}\rho(x,\eta_x(u))^3
\lesssim n^{q+1}\delta^3,
\ea
\ee
for any $u\in\mathbb{S}_x(\delta)$.
Let $\mathbf{g}_1,\mathbf{g}_2$ be the metric tensors associated with the exponential maps $\varepsilon_x : \mathbb{T}_x(\mathbb{X})\to \mathbb{X}$ and $\overline{\varepsilon}_x: \mathbb{T}_x(\mathbb{X})\to \mathbb{S}_x$, respectively. Then we have the following change of variables formulas (cf. Table~\ref{tab:balldef}):
\be\label{eq:changevars}
\int_{\mathbb{B}(x,\delta)}d\mu^*(\varepsilon_x(v))=\int_{B_{\mathbb{T}}(x,\delta)}\sqrt{|\mathbf{g}_1|}dv,\qquad \int_{\mathbb{S}_x(\delta)}d\mu^*_q(u)=\int_{B_{\mathbb{T}}(x,\delta)}\sqrt{|\mathbf{g}_2|}d\overline{\varepsilon}^{-1}_x(u).
\ee
We set $v=\overline{\varepsilon}^{-1}_x(u)$ and use the fact (cf. \eqref{eq:volume_element})  that on $B_{\mathbb{T}}(x,\delta)$, $|\sqrt{|\mathbf{g}_1|}-1|\lesssim\delta^2$ and $|\sqrt{|\mathbf{g}_2|}-1|\lesssim\delta^2$. Then by applying Equations~\eqref{eq:phidiff},~\eqref{eq:changevars},~\eqref{eq:sphkernloc},~and~\eqref{eq:sph_bernstein}, we can deduce
\be\label{eq:lemma_7_1_main_est}
\begin{aligned}
\Biggl|\int_{\mathbb{B}(x,\delta)}\Phi_{n,q}(x\cdot y)& g(y)d\mu^*(y)-\int_{\mathbb{S}_x(\delta)}\Phi_{n,q}(x\cdot u)G(u)d\mu^*_{\mathbb{S}_x}(u)\Biggr| \\
\leq&\abs{\int_{B_{\mathbb{T}}(x,\delta)}\Phi_{n,q}(x\cdot \varepsilon_x(v))g(\varepsilon_x(v))(\sqrt{|\mathbf{g}_1|}-\sqrt{|\mathbf{g}_2|})dv}\\
&\qquad+\abs{\int_{\mathbb{S}_x(\delta)}(\Phi_{n,q}(x\cdot \eta_x(u))-\Phi_{n,q}(x\cdot u))G(u)d\mu^*_{\mathbb{S}_x}(u)}\\
\lesssim&\norm{g}_\mathbb{X}(\delta^{q+2}n^q+\delta^{q+3}n^{q+1}) \le \delta^{q+3}n^{q+1}\norm{g}_\mathbb{X}(1/(n\delta)+1)
\lesssim\norm{g}_\mathbb{X}n^{-\gamma}.
\end{aligned}
\ee
Now it only remains to examine the terms away from $\mathbb{S}_x(\delta),\mathbb{B}(x,\delta)$. Utilizing Lemma~\ref{lemma:manifold_lebesgue_number}, and the fact that $S\geq \frac{2q+3\gamma}{2-\gamma}$, we have
\be\label{eq:lemma_7_1_awaymanifold}
  \abs{\int_{\mathbb{X}\setminus\mathbb{B}(x,\delta)}\Phi_{n,q}(x\cdot y)g(y)d\mu^*(y)}\lesssim\norm{g}_\mathbb{X}(n\delta)^{q-S}=\norm{g}_\mathbb{X} n^{(q-S)(2-\gamma)/(q+3)}\ls \norm{g}_{\mathbb{X}}n^{-\gamma}.
\ee
Similarly, again using Lemma~\ref{lemma:manifold_lebesgue_number} (with $\mathbb{S}_x$ as the manifold) and observing $\norm{g}_\mathbb{X}=\norm{G}_{\mathbb{S}_x}$, we can conclude
\be\label{eq:lemma_7_1_awaysphere}
\abs{\int_{\mathbb{S}_x\setminus \mathbb{S}_x(\delta)} \Phi_{n,q}(x\cdot u)G(u)d\mu_{\mathbb{S}_x}^*(u)}\lesssim \norm{G}_{\mathbb{S}_x}(n\delta)^{q-S}
\lesssim\norm{g}_\mathbb{X}n^{-\gamma},
\ee
completing the proof.
\qed

We are now in a position to complete the proof of Theorem~\ref{theo:manifold_summop}.\\[1ex]

\noindent\textsc{Proof  of Theorem~\ref{theo:manifold_summop}.}

Let $x\in \mathbb{X}$. Choose $\phi\in C^\infty$ such that $0\leq \phi(y)\leq 1$ for all $y\in \mathbb{X}$, $\phi(y)=1$ on $\mathbb{B}(x,\iota^*/2)$, and $\phi(y)=0$ on $\mathbb{X}\backslash \mathbb{B}(x,\iota^*)$. Then $f\phi$ is supported on $\mathbb{B}(x,\iota^*)$ and $F(u)\coloneqq \phi(\eta_x(u))f(\eta_x(u))$ belongs to $W_\gamma (\mathbb{S}_x)$. We observe that $\norm{f}_\mathbb{X}\lesssim \norm{f}_{W_\gamma(\mathbb{X})}$. By Lemma~\ref{lemma:manifold_lebesgue_number},
\be
\abs{\int_\mathbb{X}\Phi_{n,q}(x\cdot y)f(y)(1-\phi(y))d\mu^*(y)}\leq\norm{f}_\mathbb{X}\int_{\mathbb{X}\setminus\mathbb{B}(x,\iota^*/2)}\abs{\Phi_{n,q}(x\cdot y)}d\mu^*(y)\lesssim n^{-\gamma}\norm{f}_{W_\gamma(\mathbb{X})}.
\ee
Note that the constant above is chosen to account for the case where $n<2/\iota^*$. By Lemma~\ref{lemma:manifold_keylemma},
\begin{equation}
    \abs{\int_\mathbb{X} \Phi_{n,q}(x\cdot y)f(y)\phi(y)d\mu^*(y)-\sigma_n(\mathbb{S}_x,F)(x)}\lesssim n^{-\gamma}\norm{f\phi}_{\mathbb{X}}\lesssim n^{-\gamma}\norm{f}_{W_\gamma(\mathbb{X})}.
\end{equation}
Observe that since $f(x)=F(x)$ and $\norm{F}_{W_\gamma(\mathbb{S}_x)}\leq\norm{f}_{W_\gamma(\mathbb{X})}$,
\be
\ba
|f(x)&-\sigma_n(\mathbb{X},f)(x)|\\
&\leq |f(x)-F(x)|+|F(x)-\sigma_n(\mathbb{S}_x,F)(x)|+|\sigma_n(\mathbb{S}_x,F)(x)-\sigma_n(\mathbb{X},f)(x)|\\
&\lesssim 0+n^{-\gamma}\norm{F}_{W_\gamma(\mathbb{S}_x)}+\abs{\sigma_n(\mathbb{S}_x,F)(x)-\int_{\mathbb{X}}\Phi_{n,q}(x\cdot y)f(y)\phi(y)d\mu^*(y)}+\abs{\int_\mathbb{X}\Phi_{n,q}(x\cdot y)f(y)(1-\phi(y))d\mu^*(y)}\\
&\leq n^{-\gamma}\norm{f}_{W_\gamma(\mathbb{X})}.
\ea
\ee
Since this bound is independent of $x$, the proof is completed.
\qed

\subsection{Discretization}
\label{bhag:resultsproof}

In order to complete the proof of Theorem~\ref{theo:mainthm}, we need to discretize the integral operator in Theorem~\ref{theo:manifold_summop} while keeping track of the error.
If the manifold were known and we could use the eigendecomposition of the Laplace-Beltrami operator, we could do this discretization  without losing the accuracy using quadrature formulas (cf., e.g., \cite{mhaskardata}).
In our current set up, it is more natural to use
 concentration inequalities. 
 We will  use the inequality summarized in Proposition~\ref{prop:berncon} below (c.f.~\cite{concentration}).

\begin{proposition}[Bernstein concentration inequality]\label{prop:berncon}
Let $Z_1,\cdots, Z_M$ be independent real valued random variables such that for each $j=1,\dots,M$, $|Z_j|\leq R$, and $\mathbb{E}(Z_j^2)\leq V$. Then for any $t>0$,
\begin{equation}
    \operatorname{Prob}\left(\abs{\frac{1}{M}\sum_{j=1}^M (Z_j-\mathbb{E}(Z_j))}\geq t\right)\leq 2\exp\left(-\frac{Mt^2}{2(V+Rt/3)}\right).
\end{equation}
\end{proposition}
In the following, we will set $Z_j(x)=z_j\Phi_{n,q}(x\cdot y_j)$, where $(y_j,z_j)$ are sampled from  $\tau$.
The following lemma estimates the variance of $Z_j$.

\begin{lemma}\label{lemma:variance_est}
With the setup from Theorem~\ref{theo:mainthm}, we have
\begin{equation}\label{eq:variance}
    \sup_{x\in \mathbb{X}}\int\abs{z\Phi_{n,q}(x\cdot y)}^2d\tau(y,z)\lesssim n^q \norm{z}^2\norm{f_0}_\mathbb{X}, \qquad x\in\SS^Q.
\end{equation}
\end{lemma}

\begin{proof}\ 

We observe that \eqref{eq:sphkernloc} and Lemma~\ref{lemma:manifold_lebesgue_number} imply that
\be
\sup_{x\in\mathbb{X}}\int_{\mathbb{X}}\Phi_{n,q}(x\cdot y)^2d\mu^*(y)\lesssim n^q\sup_{x\in\mathbb{X}}\int_\mathbb{X}\abs{\Phi_{n,q}(x\cdot y)}d\mu^*(y)\lesssim n^q.
\ee
Hence,
\begin{equation}\label{eq:lem2p1}
    \sup_{x\in\mathbb{X}}\int\abs{z(y,\epsilon)\Phi_{n,q}(x\cdot y)}^2d\tau(y,z)\leq \norm{z}^2\norm{f_0}_\mathbb{X}\sup_{x\in\mathbb{X}}\int_\mathbb{X}\Phi_{n,q}(x\cdot y)^2d\mu^*(y)
    \lesssim n^q \norm{z}^2\norm{f_0}_\mathbb{X}.
\end{equation}
\end{proof}

A limitation of the Bernstein concentration inequality is that it only considers a single $x$ value. Since we are interested in supremum-norm bounds, we must first relate the supremum norm of $Z_j$ over all $x\in \mathbb{S}^Q$ to a finite set of points. We set up the connection in the following lemma.

\begin{lemma}\label{lemma:supnorm_concentration}
Let $\nu$ be any (bounded variation) measure on $\mathbb{X}$. Then there exists a finite set $\mathcal{C}$ of size $|\mathcal{C}|\sim n^{Q}$ such that
\be\label{eq:polybound}
\norm{\int_{\mathbb{X}}\Phi_{n,q}(\circ\cdot y)d\nu(y)}_{\mathbb{S}^Q}\leq 2\max_{x\in \mathcal{C}}\abs{\int_{\mathbb{X}}\Phi_{n,q}(x\cdot y)d\nu(y)}.
\ee
\end{lemma}

\begin{proof}\ 

In view of the Bernstein  inequality for the derivatives of spherical polynomials, we see that
\be\label{eq:polybern}
\abs{P(x)-P(y)}\leq cn\norm{x-y}_{\infty}\norm{P}_\infty, \qquad P\in\Pi_n^Q.
\ee
We can see by construction that $\int_{\mathbb{X}}\Phi_{n,q}(t\cdot y)d\nu(y)$ is a polynomial of degree $<n$ in the variable $t$. Since $\mathbb{S}^Q$ is a compact space and polynomials of degree $<n$ are continuous functions, there exists some $x^*\in\mathbb{S}^Q$ such that
\be
\norm{\int_{\mathbb{X}}\Phi_{n,q}(\circ\cdot y)d\nu(y)}_{\mathbb{S}^Q}=\abs{\int_{\mathbb{X}}\Phi_{n,q}(x^*\cdot y)d\nu(y)}.
\ee
Let $c$ be the same as in \eqref{eq:polybern} and $\mathcal{C}$ be a finite set satisfying
\be\label{eq:meshnorm}
\max_{x\in \mathbb{S}^Q}\min_{y\in \mathcal{C}}\norm{x-y}_\infty\leq \frac{1}{2cn}.
\ee
Since $\mathbb{S}^Q$ is a compact $Q$-dimensional space, the set $\mathcal{C}$  needs no more than $\sim n^{Q}$ points.

Then there exists some $z^*\in\mathcal{C}$ such that
\be
\abs{\int_\mathbb{X}(\Phi_{n,q}(x^*\cdot y)-\Phi_{n,q}(z^*\cdot y))d\nu(y)}\lesssim n\norm{x^*-z^*}_\infty\abs{\int_{\mathbb{X}}\Phi_{n,q}(x^*\cdot y)d\nu(y)},
\ee
which implies \eqref{eq:polybound}.

\end{proof}

With this preparation, we now state the following theorem which gives a bound on the difference between our discrete approximation $F_n$ and continuous approximation $\sigma_n$ with high probability.

\begin{theorem}\label{theo:theorem2} Assume the setup of Theorem~\ref{theo:mainthm}. Then for every $n\geq 1$ and $M\gtrsim n^{q+2\gamma}\log(n/\delta)$ we have
\be\label{eq:thm2}
\operatorname{Prob}_\tau\left(\norm{F_n(\mathcal{D};\circ)-\sigma_n(\mathbb{X},ff_0)}_{\mathbb{S}^Q}\geq c \norm{z}n^{-\gamma}\sqrt{\norm{f_0}_\mathbb{X}}\right)\leq \delta.
\ee
\end{theorem}

\begin{proof}
In this proof only, constants $c,c_1,c_2,\dots$ will maintain their value once used. Let $Z_j(x)=z_j\Phi_{n,q}(x\cdot y_j)$. Since $z$ is integrable with respect to $\tau$, one has the following for any $x\in \mathbb{S}^Q$:
\be
\mathbb{E}_\tau(Z_j(x))=\int_{\mathbb{X}} \mathbb{E}_\tau(z|y)\Phi_{n,q}(x\cdot y)d\nu^*(y)=\int_{\mathbb{X}}f(y)\Phi_{n,q}(x\cdot y)f_0(y)d\mu^*(y)=\sigma_n(\mathbb{X},ff_0)(x).
\ee
We have from \eqref{eq:sphkernloc} that $\abs{Z_j}\lesssim n^q\norm{z}$. Lemma~\ref{lemma:variance_est} informs us that $\mathbb{E}_\tau(Z_j^2)\lesssim n^q\norm{z}^2\norm{f_0}_\mathbb{X}$. Assume $0<r\leq 1$ and set $t=r\norm{z}\norm{f_0}_\mathbb{X}$. From Proposition~\ref{prop:berncon}, we see
\be\label{eq:thm20}
\begin{aligned}
\operatorname{Prob}_\tau\left(\abs{\frac{1}{M}\sum_{j=1}^M Z_j(x)-\sigma_n(\mathbb{X},ff_0)(x)}\geq t\right)\leq& 2\exp\left(-c_1\frac{Mt^2}{(n^q\norm{z}^2\norm{f_0}_\mathbb{X}+n^q\norm{z}t/3)}\right)\\
\leq&2\exp\left(-c_2\frac{M\norm{f_0}_\mathbb{X}r^2}{n^q}\right).
\end{aligned}
\ee
Let $\delta\in (0,1/2)$, $\mathcal{C}$ be a finite set satisfying~\eqref{eq:meshnorm} with $\abs{\mathcal{C}}\leq c_3n^{Q}$ (without loss of generality we assume $c_3\geq 1$),
\be\label{eq:creq}
c_4\geq \frac{\max\big(\log_2(c_3)+1,Q\big)}{c_2},
\ee
and
\be\label{eq:Mreq}
M\geq c_4n^{q+2\gamma}\log(n/\delta).
\ee
We now fix
\be\label{eq:rdef}
r\equiv\sqrt{c_4\frac{n^q}{M\norm{f_0}_\mathbb{X}}\log(n/\delta)}.
\ee
Notice that since $\norm{f_0}_\mathbb{X}\geq 1$, our assumption of $M$ in \eqref{eq:Mreq} implies
\be
r\leq n^{-\gamma}\Big/\sqrt{\norm{f_0}_\mathbb{X}}\leq 1,
\ee
so our choice of $r$ may be substituted into \eqref{eq:thm20}. Further,
\be\label{eq:r}
r\norm{z}\norm{f_0}_\mathbb{X}\leq c\norm{z} n^{-\gamma}\sqrt{\norm{f_0}_\mathbb{X}}.
\ee
With this preparation, we can conclude
\be\label{eq:thm21}
\begin{aligned}
\operatorname{Prob}_\tau\bigg(\big\|F_n(\mathcal{D};\circ)&-\sigma_n(\mathbb{X},ff_0)\big\|_{\mathbb{S}^Q}\geq c \norm{z}n^{-\gamma}\sqrt{\norm{f_0}_\mathbb{X}}\bigg)&\\
\leq&\operatorname{Prob}_\tau\left(\norm{\frac{1}{M}\sum_{j=1}^M Z_j-\sigma_n(\mathbb{X},ff_0)}_{\mathbb{S}^Q}\geq r\norm{z}\norm{f_0}_\mathbb{X}\right)&\text{(from \eqref{eq:r})}\\
\leq&\operatorname{Prob}_\tau\left(\max_{x_k\in \mathcal{C}}\left(\abs{\frac{1}{M}\sum_{j=1}^M Z_j(x_k)-\sigma_n(\mathbb{X},ff_0)(x_k)}\right)\geq t \right)&\text{(by Lemma~\ref{lemma:supnorm_concentration})}\\
\leq&\sum_{k=1}^{\abs{\mathcal{C}}}\operatorname{Prob}_\tau\left(\abs{\frac{1}{M}\sum_{j=1}^M Z_j(x_k)-\sigma_n(\mathbb{X},ff_0)(x_k)}\geq t \right)&\\
\leq& |\mathcal{C}|\exp\left(-c_2\frac{M\norm{f_0}_\mathbb{X}r^2}{n^q}\right)&\text{(from \eqref{eq:thm20})}\\
 \leq& c_3n^{Q-c_2c_4}\delta^{c_2c_4}&\text{(from \eqref{eq:rdef})}\\
 \leq&c_3n^{Q-Q}\left(\frac{1}{2}\right)^{\log_{1/2}(1/c_3)}\delta&\text{(from \eqref{eq:creq} and $\delta<1/2$)}\\
 \leq&\delta.
\end{aligned}
\ee
\end{proof}

We are now ready for the proof of Theorem~\ref{theo:mainthm}.

\begin{proof}[Proof of Theorem~\ref{theo:mainthm} (and Corollary~\ref{cor:maincor}~and~\ref{cor:densityest}).]
Since $f,f_0\in W_\gamma(\mathbb{X})$, we can determine that $ff_0\in W_\gamma(\mathbb{X})$ as well. Utilizing Theorem~\ref{theo:manifold_summop} with $ff_0$ and Theorem~\ref{theo:theorem2}, we obtain with probability at least $1-\delta$ that
\be
\ba
\norm{F_n(\mathcal{D};\circ)-ff_0}_{\mathbb{X}}\leq&\norm{F_n(\mathcal{D};\circ)-\sigma_n(\mathbb{X};ff_0)}_\mathbb{X}+\norm{\sigma_n(\mathbb{X};ff_0)-ff_0}_\mathbb{X}\\
\lesssim&\frac{\sqrt{\norm{f_0}_{\mathbb{X}}}\norm{z}+\norm{ff_0}_{W_\gamma(\mathbb{X})}}{n^\gamma}.
\ea
\ee
Corollary~\ref{cor:maincor} is seen immediately by setting $f_0=1$. Corollary~\ref{cor:densityest} follows from setting $z=1$ and then observing that $f=1$ and $\sqrt{\norm{f_0}_\mathbb{X}}\lesssim \norm{f_0}_{W_\gamma(\mathbb{X})}$.
\end{proof}

\section{Conclusions}\label{sec:conclusions}
We have discussed a central problem of machine learning, namely to approximate an unknown target function based only on the data drawn from an unknown probability distribution.
While the prevalent paradigm to solve this problem in general is to minimize a loss functional, we have initiated a new paradigm where we can do the approximation directly from the data, under the so-called manifold assumption.
This is a substantial theoretical improvement over the classical manifold learning technology, which involves a two-step procedure: first to get some information about the manifold and then to do the approximation.
Our method is a ``one-shot'' method that bypasses collecting any information about the manifold itself: it learns on the manifold without manifold learning.
Our construction in itself does not require any assumptions on the probability distribution or the target function.
We derive uniform error bounds with high probability regardless of the nature of the distribution, provided we know the dimension of the unknown manifold.
The theorems are illustrated with some numerical examples. 
One of these is closely related to an important problem in magnetic resonance relaxometry, in which one seeks to find the proportion of water molecules in the myelin covering in the brain based on a model that involves the inversion of Laplace transform.

We view our paper as the beginning of a new direction. 
As such, there are plenty of future research projects, some of which we plan to undertake ourselves.
\begin{itemize}
\item Find alternative methods that improve upon the error estimates on \textbf{unknown} manifolds, and more general compact sets.
The encoding described in Section~\ref{sec:encoding} gives a representation of a function on an unknown manifold.
Such an encoding is useful in the emerging area of approximation of operators. 
It is clear that the encoding described in Section~\ref{sec:encoding} for functions on manifolds itself forms a submanifold of a Euclidean space, which in turn can be projected to a submanifold of a sphere.
We plan to develop this theme further in the context of approximation of operators defined in different function spaces.
\item Explore real-life applications other than the examples which we have discussed in this paper.
\item We feel that our method will work best if we are working in the right feature space.
One of the vexing problems in machine learning is to identify the right features in the data.
Deep networks are supposed to be doing this task automatically.
However, there is no clear explanation of whether they work in every problem or otherwise develop a theory of what ``features'' should mean and how deep networks can extract these automatically.
\end{itemize}

\nc{\mathcal{D},\tau}{Set of data $\mathcal{D}=\{y_j,z_j\}_{j=1}^M$ sampled from distribution $\tau$. It is assumed the $y_j$'s lie on a $q$-dimensional submanifold of $\mathbb{S}^Q$.}
\nc{\mathbb{S}^q,\mu^*_q}{Sphere of $q$-dimensions with probability measure $\mu^*_q$ as defined in \eqref{eq:sphere}.}
\nc{\omega_q}{Volume of the $q$-dimensional sphere.}
\nc{\mathbb{Y}}{Equator of $\mathbb{S}^Q$, as defined in Section~\ref{subsec:eqapprox}.}
\nc{\mathbb{X},\rho,\mu^*}{Submanifold of $\mathbb{S}^Q$ with geodesic $\rho$ and normalized volume measure $\mu^*$.}
\nc{Q,q}{Ambient dimension of the data and dimension of the underlying manifold, respectively.}
\nc{p_{q,\ell}}{orthonormalized ultraspherical polynomial of degree $\ell$ and dimension $q$ as defined in Section~\ref{subsec:sphharm}.}
\nc{\Phi_{n,q}}{Localized kernels as defined in \eqref{eq:kernel}.}
\nc{\sigma_n}{Continuous approximation (a.k.a. integral reconstruction) operator as defined on $\mathbb{S}^q$ in \eqref{eq:sphapproxop}, $\mathbb{Y}$ in \eqref{eq:sphapproxop2}, and $\mathbb{X}$ in \eqref{eq:manifold_summabilityop}.}
\nc{F_n}{Our proposed constructive approximation for finite data as defined in \eqref{eq:approximation}.}
\nc{Y_{\ell,k}}{Basis elements for the space of homogenous, harmonic polynomials of degree $\ell$.}
\nc{\mathbb{H}_\ell^q}{Space of homogenous, harmonic polynomials of degree $\ell$ in $q$ dimensions.}
\nc{\Pi_n^q}{Space of spherical polynomials of degree $<n$. See Section~\ref{subsec:sphharm}.}
\nc{E_n}{Degree of approximation as defined in Sections~\ref{subsec:sphapprox}~and~\ref{subsec:eqapprox}.}
\nc{W_\gamma}{Smoothness class of functions as defined on $\mathbb{S}^q$ in Section~\ref{subsec:sphapprox}, on $\mathbb{Y}$ in Section~\ref{subsec:eqapprox}, and on $\mathbb{X}$ in Definition~\ref{def:manifold_smoothness}.}
\nc{B_{Q+1}(x,r),S^Q(x,r),B_\mathbb{T},\mathbb{S}_x,\mathbb{B}}{Balls on various spaces. See Table~\ref{tab:balldef} for reference.}
\nc{\mathbb{S}_x}{The unique $q$-dimensional equator of $\mathbb{S}^Q$ that shares a tangent space $\mathbb{T}_x$ with the point $x\in\mathbb{X}$.}
\nc{\varepsilon_x,\overline{\varepsilon}_x}{Exponential map for $\mathbb{X},\mathbb{S}_x$ respectively. See \ref{sec:manifoldintro} for details.}
\nc{\eta_x}{Composite map $\varepsilon_x\circ \overline{\varepsilon}^{-1}_x$ defined in Section~\ref{sec:manifoldapprox}.}
\nc{\iota_1,\iota_2,\iota^*}{Injectivity radii of $\varepsilon_x,\overline{\varepsilon}_x,\eta_x$, respectively.}
\nc{\mathbf{g}_1,\mathbf{g}_2}{Metric tensors associated with $\varepsilon_x,\overline{\varepsilon}_x$, respectively. See \ref{sec:manifoldintro} for details.}

\printnomenclature

\appendix

\section*{Appendix}

\renewcommand{\thesection}{A}
\setcounter{equation}{0}
\renewcommand{\theequation}{\thesection.\hindu{equation}}

\subsection{Encoding}
\label{sec:encoding}

Our construction in \eqref{eq:approximation} allows us to encode the target function in terms of finitely many real numbers.
For each integer $\ell\geq 0$, let $\{Y_{Q,\ell,k}\}_{k=1}^{\operatorname{dim}(\mathbb{H}_\ell^Q)}$ be an orthonormal basis for $\mathbb{H}_\ell^Q$ on $\mathbb{S}^Q$. We define  the encoding of $f$ by
\be\label{eq:encoding}
\hat{z}(\ell,k)\coloneqq \frac{1}{M}\sum_{j=1}^M z_jY_{Q,\ell,k}(y_j).
\ee
Given this encoding, the decoding algorithm is given in the following proposition.
\begin{proposition}\label{prop:encoding}
Assume $\Phi_{n,q}$ is given as in \eqref{eq:kernel}. Given the encoding of $f$ as given in \eqref{eq:encoding}, one can rewrite the approximation in \eqref{eq:approximation} as
\be\label{eq:decoder}
F_n(\mathcal{D};x)=\sum_{\ell=0}^n \Gamma_{\ell,n}\sum_{k=1}^{\operatorname{dim}(\mathbb{H}_\ell^Q)}\hat{z}(\ell,k)Y_{Q,\ell,k}(x)\qquad x\in \mathbb{S}^Q,
\ee 
where
\be\label{eq:connection}
\Gamma_{\ell,n}\coloneqq \frac{\omega_q\omega_{Q-1}}{\omega_Q\omega_{q-1}}\sum_{i=\ell}^n h\left(\frac{i}{n}\right)\frac{p_{q,i}(1)}{p_{Q,\ell}(1)}C_{Q,q}(\ell,i),
\ee
and $C_{Q,q}(\ell,i)$ is defined in \eqref{eq:dimchange}.
\end{proposition}

\begin{proof}
The proof follows from writing out
\be
F_n(\mathcal{D};x)=\frac{\omega_q}{M\omega_{q-1}}\sum_{j=1}^M z_j \sum_{i=1}^n h\left(\frac{i}{n}\right)p_{q,i}(1)p_{q,i}(x\cdot y_j),
\ee
making substitutions using \eqref{eq:dimchange}, \eqref{eq:reproduce}, and  \eqref{eq:summation}, then collecting terms.
\end{proof}

\begin{rem}{\rm
The encoding \eqref{eq:encoding} is not parsimonious. Since the basis functions $\{Y_{Q,\ell,k}\}_{\ell=0,k=1}^{n,\operatorname{dim}(\mathbb{H}_\ell^Q)}$ is not necessarily independent on $\XX$, the encoding can be made parsimonious by exploiting linear relationships in this system. Given a reparametrization the functions as $\{Y_j\}_{j=1}^{\sum_{\ell=0}^n \operatorname{dim}(\mathbb{H}_\ell^Q)}$, we  form the discrete Gram matrix $G$ by the entries
\be\label{eq:discrete_gram}
G_{i,j}\coloneqq \frac{1}{M}\sum_{k=1}^M Y_i(y_k)Y_j(y_k) \approx\int_\mathbb{X} Y_i(y) Y_j(y) f_0d\mu^*(y).
\ee
In practice, one may formulate a QR decomposition by fixing some first basis vector and proceeding by the Gram-Schmidt process until a basis is formed, then setting some threshold on the eigenvalues to get the desired dependencies among the $Y_j$'s. \qed}
\end{rem} 

\subsection{Background on manifolds}
\label{sec:manifoldintro}

This introduction to manifolds covers the main ideas which we use in this paper without going into much detail. We mostly follow along with the notation and definitions in \cite{docarmo}. For details, we refer the reader to  texts such as \cite{boothby,docarmo,guillemin}.

\begin{definition}[Differentiable Manifold]
A (boundary-less) \textit{differentiable manifold} of dimension $q$ is a set $\mathbb{X}$ together with a family of open subsets $\{U_\alpha\}$ of $\mathbb{R}^q$ and functions $\{\vec{x}_\alpha\}$ such that
\be
\vec{x}_\alpha: U_\alpha \to \mathbb{X}
\ee
is injective, and the following 3 properties hold:
\begin{itemize}
\item $\bigcup_{\alpha}\vec{x}_\alpha(U_\alpha)=\mathbb{X}$,
\item $\vec{x}_\alpha(U_\alpha)\cap \vec{x}_\beta(U_\beta)=W\neq \emptyset$ implies that $\vec{x}^{-1}_\alpha(W),\vec{x}^{-1}_\beta(W)$ are open sets and $\vec{x}_\beta^{-1}\circ \vec{x}_\alpha$ is an infinitely differentiable function.
\item The family $\mathcal{A}_\mathbb{X}=\{(U_\alpha,\vec{x}_\alpha)\}$ is maximal regarding the above conditions.
\end{itemize}
\end{definition}

\begin{rem}{\rm
The pair $(U_\alpha,\vec{x}_\alpha)$ gives a \textit{local coordinate chart} of the manifold, and the collection of all such charts $\mathcal{A}_\mathbb{X}$ is known as the \textit{atlas}. \qed}
\end{rem}

\begin{definition}[Differentiable Map]
Let $\mathbb{X}_1,\mathbb{X}_2$ be differentiable manifolds. We say a function $\phi: \mathbb{X}_1\to\mathbb{X}_2$ is (infinitely) \textit{differentiable}, denoted by $\phi\in C^\infty(\mathbb{X})$, at a point $x\in\mathbb{X}_1$ if given a chart $(V,\vec{y})$ of $\mathbb{X}_2$, there exists a chart $(U,\vec{x})$ of $\mathbb{X}_1$ such that $x\in \vec{x}(U)$, $\phi(\vec{x}(U))\subseteq \vec{y}(V)$, and $\vec{y}^{-1}\circ \phi\circ \vec{x}$ is infinitely differentiable at $\vec{x}^{-1}(p)$ in the traditional sense.
\end{definition}

For any interval $I$ of $\mathbb{R}$, a differentiable function $\gamma: I\to \mathbb{X}$ is known as a $\textit{curve}$. If $x\in\mathbb{X}$, $\epsilon>0$, and $\gamma: (-\epsilon,\epsilon)\to \mathbb{X}$ is a curve with $x=\gamma(0)$, then we can define the \textit{tangent vector} to $\gamma$ at $\gamma(t_0)$ as a functional $\gamma'(t_0)$ acting on the class of differentiable functions $f:\mathbb{X}\to\mathbb{R}$ by
\be
\gamma'(t_0)f\coloneqq\frac{d(f\circ \gamma)}{dt}(t_0).
\ee
The \textit{tangent space} of $\mathbb{X}$ at a point $x\in\mathbb{X}$, denoted by $\mathbb{T}_x(\mathbb{X})$, is the set of all such functionals $\gamma'(0)$.

A \textit{Riemannian} manifold is a differentiable manifold with a family of inner products $\{\alg{\circ,\circ}_x\}_{x\in\mathbb{X}}$ such that for any $X,Y\in\mathbb{T}_x(\mathbb{X})$, the function $\varphi: \mathbb{X}\to\mathbb{C}$ given by $x\mapsto \alg{X(x),Y(x)}_x$ is differentiable. 
We can define an associated norm $\norm{X}=\alg{X(x),X(x)}_x$. 
The length $L(\gamma)$ of a curve $\gamma$ defined on $[a,b]$ is defined to be 
\be
L(\gamma)\coloneqq\int_a^b\norm{\gamma'(t)}dt.
\ee
 We will call a curve $\gamma:[a,b]\to \mathbb{X}$ a \textit{geodesic} if $L(\gamma)=\inf\{L(r): r:[a,b]\to\mathbb{X}, r\text{ is a curve}\}$. It is well-known that if $\gamma$ is a geodesic, then $\gamma'(t)\cdot \gamma''(t)=0$ for any $t\in [a,b]$.
 
 In the sequel, we assume that $\mathbb{X}$ is a compact, connected, Riemannian manifold.
Then for every $x,y\in\mathbb{X}$ there exists a geodesic $\gamma:[a,b]\to\mathbb{X}$  such that $\gamma(a)=x,\gamma(b)=y$. The quantity  $\rho(x,y)=L(\gamma)$ defines a metric on $\XX$ such that the corresponding metric topology is consistent with the topology defined by any atlas on $\XX$.

For any $x\in \XX$, there exists a neighborhood $V\subset \XX$ of $x$, a number $\delta=\delta(x)>0$ and a mapping $\mathcal{E} : (-2,2)\times U \to \XX$, where $U=\{(y,v) : y\in V,\  v\in T_y\XX,\ \|v\|_2<\delta\}$ such that $t\mapsto \mathcal{E}(t, y, v)$ is the unique geodesic of $\XX$ which, at $t=0$, passes through $y$ and has the property that $\partial{\mathcal{E}}/\partial t =v$ for each $(y,v)\in U$.
As a result, we can define the \textit{exponential map} at $x$ to be the function $\varepsilon_{x}: B_\mathbb{T}(x,\delta(x))\subset \mathbb{T}_x(\mathbb{X})\to \mathbb{X}$ by $\varepsilon_x(v)= \mathcal{E}(1, x, v)$. 
Intuitively, the line joining $x$ and $v$ in $\mathbb{T}_x(\mathbb{X})$ is mapped to the geodesic joining $x$ with $\varepsilon_x(v)$. 
We call the supremum of all $\delta(x)$ for which the exponential map is so defined  the \textit{injectivity radius} at $x$, denoted by $\iota(x)$.
We call $\iota^*=\inf_{x\in\mathbb{X}}\iota(x)$ the  \textit{global injectivity radius} of $\mathbb{X}$. 
Since $x\mapsto \iota(x)$ is a continuous function of $x$, and $\iota(x)>0$ for each $x$,  it follows that $\iota^*>0$ when $\mathbb{X}$ is compact. 
Correspondingly, on compact manifolds, one can conclude that for $y\in B_\mathbb{T}(x,\iota^*)$, $\rho(x,\varepsilon_x(y))=\norm{x-y}$.

Next, we discuss the metric tensor and volume element on $\XX$.
Let $(U,\vec{x})$ be a coordinate chart with $0\in U$, $\vec{x}(0)=x\in\mathbb{X}$, and $\partial_j(x)$ be the tangent vector at $x$ to the coordinate curve $t\mapsto \vec{x}((\underbrace{0,\dots,0}_{j-1},t,0,\dots,0))$. Then we can define the metric tensor $\mathbf{g}$ to be the matrix where $\mathbf{g}_{ij}=\alg{\partial_i(x),\partial_j(x)}_x$. When one expands the metric tensor $\mathbf{g}$ as a Taylor series in local coordinates on $\mathbb{B}(x,\iota^*)$, it can be shown \cite[pg. 21]{roemanifold} that for any $\delta<\iota^*$, on the ball $\mathbb{B}(x,\delta)$ we have
\be
|\mathbf{g}|=1+O(\delta^2).
\ee
In turn, this implies
\be\label{eq:volume_element}
\sqrt{|\mathbf{g}|}-1\lesssim \delta^2.
\ee

The following proposition lists some important properties relating the geodesic distance $\rho$ on an unknown submanifold of $\SS^Q$ with the geodesic distance on $\SS^Q$ as well as the Euclidean distance on $\RR^{Q+1}$.
\begin{proposition}\label{prop:taylorprop} Let $\eta_x$ be defined as in Section~\ref{sec:manifoldapprox}.

{\rm (a)} For every $\eta_x(u)\in\mathbb{B}(x,\iota^*)$,
\be\label{eq:diffbound}
\abs{\arccos(x\cdot \eta_x(u))-\rho(x,\eta_x(u))}\lesssim\rho(x,\eta_x(u))^3.
\ee

{\rm (b)} For any $x,y\in\mathbb{X}$,
\be
\rho(x,y)\sim \arccos(x\cdot y).
\ee

\end{proposition}

\begin{proof}
First, we observe the fact that $\norm{x-y}_2\sim \arccos(x\cdot y)$ because $\norm{x-y}_2/2=\sin(\arccos(x\cdot y)/2)$ and $\theta/\pi \le \sin(\theta/2)\le \theta/2$ for all $\theta\in [0,\pi]$. Fix $x\in\mathbb{X}$ and let $\gamma$ be a geodesic on $\mathbb{X}$ parametrized by length $t$ from $x$. In particular we then have $\norm{\gamma'(0)}_2=1$ and $\gamma'(0)\cdot\gamma''(0)=0$.
 Taking a Taylor expansion for $\gamma(t)$ with $|t|<\iota^*$ (we recall that $\iota^*\leq 1)$, we can see
\be\label{eq:taylorgamma}
\ba
\gamma'(0)\cdot (\gamma(t)-\gamma(0))=&\gamma'(0)\cdot \left(\gamma'(0)t+\frac{1}{2}\gamma''(t)t^2+O(t^3)\right)\\
=&\norm{\gamma'(0)}^2_2t+\gamma'(0)\cdot \gamma''(0)t^2+O(t^3)\\
=&t+O(t^3).
\ea
\ee
For any $y\in\mathbb{B}(x,\iota^*)$, there exists a unique $u\in \mathbb{S}_x(\iota^*)$ such that $y=\eta_x(u)$. We can write $y=\gamma(t)$ for some geodesic $\gamma$. We know, $t=\rho(x,y)\geq \arccos(x\cdot y)\geq \norm{x-y}_2=\norm{\gamma(t)-\gamma(0)}_2$. Using the Cauchy-Schwarz inequality, we see
\be
0\leq t-\norm{x-y}_2\leq t-\gamma'(0)\cdot (\gamma(t)-\gamma(0))\lesssim t^3.
\ee
As a result we can conclude
\be
\rho(x,\eta_x(u))-\arccos(x\cdot \eta_x(u))\leq \rho(x,\eta_x(u))-\norm{\eta_x(u)-x}_2\lesssim \rho(x,\eta_x(u))^3,
\ee
showing \eqref{eq:diffbound}. Letting $c$ be the constant built into the notation of \eqref{eq:diffbound}, then if we fix $x\in\mathbb{X}$ and let $y\in\mathbb{B}\big(x,\sqrt{1/(2c)}\big)$, we have
\be
\frac{1}{2}\rho(x,y)\leq \rho(x,y)-c\rho(x,y)^3\leq \arccos(x\cdot y).
\ee
Furthermore, since $A=\overline{\mathbb{X}\setminus\mathbb{B}\big(x,\sqrt{1/(2c)}\big)}$ is a compact set and $g_x(y)=\arccos(x\cdot y)/\rho(x,y)$ is a continuous function of $y$ defined on $A$, we can conclude that $g_x$ attains a minimum on $A$. Therefore,
\be
\rho(x,y)\sim \arccos(x\cdot y)
\ee
for every $y\in \mathbb{X}$.
We note that the constants involved in this proof vary continuously with respect to the choice of $x$, so in the theorem we may simply use the supremum over all such constants which must be finite since $\mathbb{X}$ is compact.
\end{proof}

\subsection{Network representation}
\label{subsec:clenshaw}

\begin{algorithm}[!ht]
\begin{algorithmic}[1]
\item[{\rm a)}] \textbf{Input:} $p_0$, $C_0,\cdots, C_{n-1}$, $x$, $a_{n+1},\cdots, a_1$, $b_{n+1},\cdots,b_1$.
\item[{\rm b)}] \textbf{Output:} The value of $\sum_{k=0}^{n-1}C_kp_k$.
\STATE $\mathsf{out1}\gets 0, \mathsf{out2}\gets 0, C_{-1} \gets 0,  C_n\gets 0$.
\FOR{$k=n+1$ down to 1}
\STATE $\mathsf{temp}\gets a_k*\mathsf{out_1}*x+\mathsf{out2}$
\STATE $\mathsf{out2}\gets b_k*\mathsf{out1}+C_{k-2}$
\STATE $\mathsf{out1}\gets \mathsf{temp}$.
\ENDFOR
\STATE \textbf{Return:} $\mathsf{out1}*p_0$.
 \end{algorithmic}
 \caption{Clenshaw algorithm to compute $\sum_{k=0}^{n-1}C_kp_k$, where\\ $p_k(x)=a_kxp_{k-1}(x)+b_kp_{k-2}(x)$, $k=1,2,\cdots,n-1$, $b_{1}=0$.}
 \label{alg:clenshaw}
 \end{algorithm}

Let $\{p_k\}$ be a system of orthonormal polynomials satisfying a recurrence relation
\be\label{eq:clenshawrec}
p_k(x)=a_kxp_{k-1}(x)+b_kp_{k-2}(x), \qquad k=1,2,\cdots, \quad b_{1}=0.
\ee
The Clenshaw algorithm is a modification of the classical Horner method to compute polynomials expressed in the monomial basis that evaluates a polynomial expressed in terms of the orthonormalized polynomials $\{p_k\}$ \cite{clenshaw,gautschibk}.
To understand the method, let $P=\sum_{k=0}^{n-1}C_kp_k$.
It is convenient to write $C_k=0$ if $k\ge n$ or $k<0$.
The recurrence \eqref{eq:clenshawrec} shows that
\be\label{eq:clenshawstep}
C_kp_k(x)+C_{k-1}p_{k-1}(x)+C_{k-2}p_{k-2}(x)=\left(a_kC_kx+C_{k-1}\right)p_{k-1}(x)+\left(b_kC_k+C_{k-2}\right)p_{k-2}.
\ee
This leads to Algorithm~\ref{alg:clenshaw}.

\begin{figure*}[!ht]
\begin{center}
\includegraphics[scale=0.1]{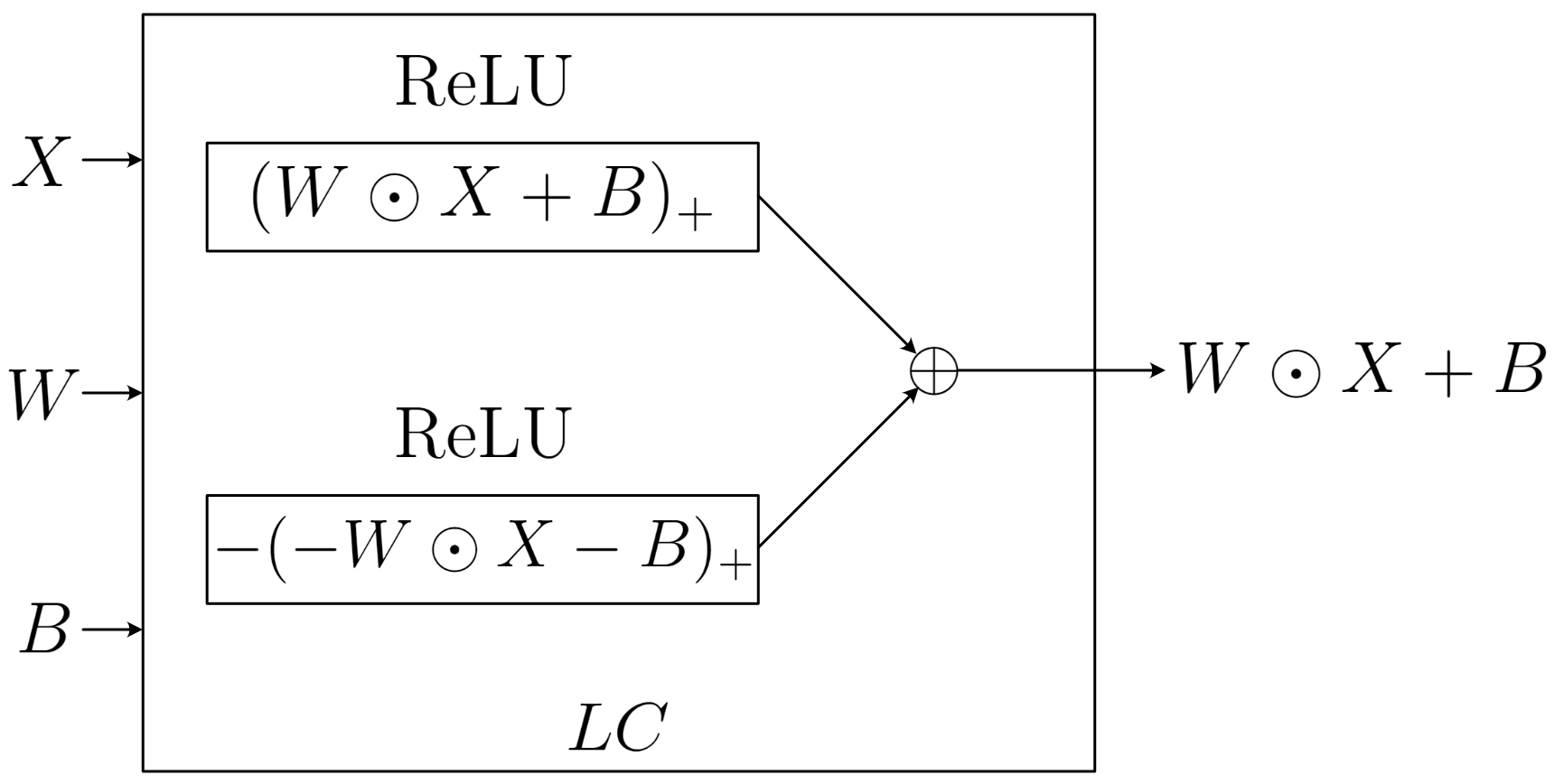} 
\end{center}
\caption{The implementation of a linear combination as a ReLU network. Here all operations are pointwise. The symbols $\odot$ represents Hadamard product of matrices, $\oplus$ is the sum of matrices.}
\label{fig:linearcomb}
\end{figure*} 

\begin{figure*}[!ht]
\begin{center}
\includegraphics[scale=0.15]{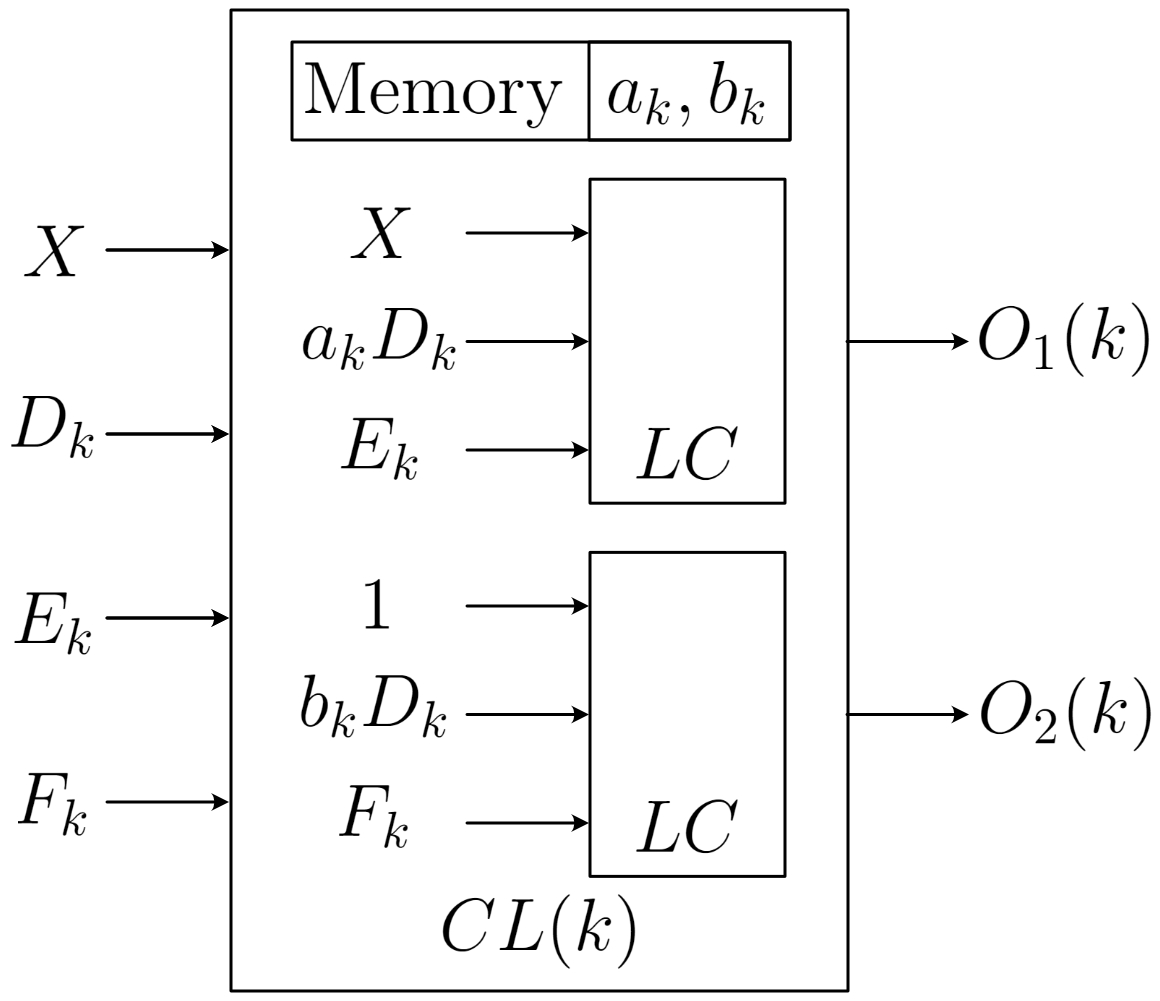} 
\end{center}
\caption{One step of the Clenshaw algorithm, using two circuits of the form LC (4 neurons) as in Figure~\ref{fig:linearcomb}. 
The circuit diagram is shown in general with four input pins and two output pins.}
\label{fig:clenshawstep}
\end{figure*}

\begin{figure*}[!ht]
\begin{center}
\includegraphics[scale=0.25]{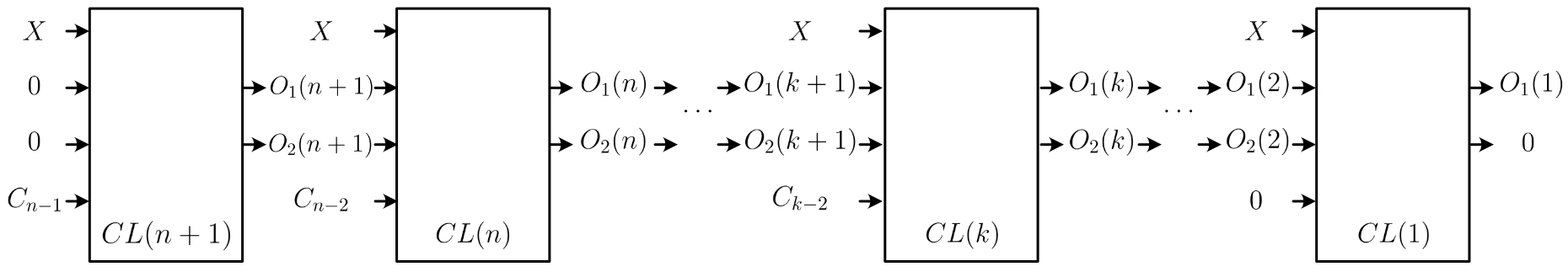} 
\end{center}
\caption{Unrolling the Clenshaw algorithm as a cascade of the circuits of the form CL($k$) as in Figure~\ref{fig:clenshawstep}.}
\label{fig:clenshawalg}
\end{figure*}
 
By algorithm unrolling, we may express this algorithm in terms of a deep neural network evaluating a ReLU activation function.
The network is a cascade of different circuits. 
The most fundamental is the implementation of a linear combination as a ReLU network (see  Figure~\ref{fig:linearcomb})
$$
ax+b=(ax+b)_+ -(-ax-b)_+.
$$
Using the circuits LC in Figure~\ref{fig:linearcomb}, we next construct a circuit to implement recursive reduction \eqref{eq:clenshawstep}.
This is illustrated in Figure~\ref{fig:clenshawstep}.
Finally, we unroll the Clenshaw algorithm by cascading the circuits CL($k$) from Figure~\ref{fig:clenshawstep} for $k=n+1$ down to $k=1$ with different inputs and outputs as shown in Figure~\ref{fig:clenshawalg}.
We use this in order to compute $\Phi_{n,q}(x\cdot y_j)$ by using the recursive formula for ultraspherical polynomials \eqref{eq:recurrence} in the following way.
We set
\be
\begin{aligned}
C_k&=\frac{\omega_q}{\omega_{q-1}}h(k/n)p_{q,k}(1),\\
a_k&= \begin{cases}\displaystyle{\frac{\sqrt{\Gamma(q)\Gamma(q+1)}}{\Gamma(q-1)}}& k=1 \vspace{5pt}\\ \displaystyle{\sqrt{\frac{(2k+q-3)(2k+q-1)}{k(n+q-2)}}}& k\geq 2\end{cases},\\
b_k&=\sqrt{\frac{(k-1)(k+q-3)(2k+q-1)}{k(k+q-2)(2k+q-5)}}.
\end{aligned}
\ee
For the matrix $X$ shown in Figure~\ref{fig:clenshawalg}, we consider the $(Q+1)\times N$ test data matrix $S$ where each column represents one test data $x$, and a $(Q+1)\times M$ train data matrix $R$ where column $j$ represents data point $y_j$. 
Then we set $X=S^TR$. In this way, we would return $\Phi_{n,q}(S^TR)$ from running Algorithm~\ref{alg:clenshaw}, with a time complexity of $O(NMn)$.

\bibliographystyle{abbrv}
\bibliography{references.bib}

\end{document}